\documentclass[conference]{IEEEtran}
\IEEEoverridecommandlockouts

\usepackage{xspace}
\usepackage{amsmath}
\usepackage{amsmath,amssymb,amsfonts,amsfonts,amsthm}
\usepackage{algorithm}
\usepackage[noend]{algorithmic}
\usepackage[caption=false,font=footnotesize]{subfig}	
\usepackage{framed} 
\usepackage[]{epsfig,color}
\usepackage[nice]{nicefrac}
\usepackage{multirow}
\usepackage{url}
\usepackage{wrapfig}
\usepackage{color}

\usepackage{xargs}
\usepackage[pdftex,dvipsnames]{xcolor}  
\usepackage[textsize=tiny]{todonotes}

\newcommand{\Cpp}{C\raise.08ex\hbox{\tt ++}\xspace}

\newcommand{\Cfree}{\ensuremath{\calX_{\text{free}}}\xspace}
\newcommand{\Cforb}{\ensuremath{\calX_{\text{forb}}}\xspace}
\newcommand{\Cs}{C-space\xspace}
\newcommand{\eRRT}{LBT-RRT\xspace}
\newcommand{\eRRTbold}{LBT-RRT\xspace}

\newcommand{\Glb}{\ensuremath{\calG_{lb}}\xspace}
\newcommand{\Tlb}{\ensuremath{\calT_{lb}}\xspace}
\newcommand{\Tub}{\ensuremath{\calT_{apx}}\xspace}

\newcommand{\set}[1]{\ensuremath{\{ #1\}}}

\newcommand{\calG}{\ensuremath{\mathcal{G}}\xspace}

\newcommand{\calX}{\ensuremath{\mathcal{X}}\xspace}

\newcommand{\calT}{\ensuremath{\mathcal{T}}\xspace}

\newtheorem{thm}{Theorem}[section]
\newtheorem{cor}[thm]{Corollary}
\newtheorem{lem}[thm]{Lemma}
\newtheorem{obs}[thm]{Observation}

\newcommand{\ignore}[1]{}

\newboolean{ICRA}
\newboolean{ARXIV}

\setboolean{ICRA}{false}
\ifthenelse{\boolean{ICRA}}
	{\setboolean{ARXIV}{false} }
	{\setboolean{ARXIV}{true}  }
\newcommand{\textVersion}[2]
{\ifthenelse{\boolean{ICRA} }{#1}{}\ifthenelse{\boolean{ARXIV}}{#2}{}}

\newboolean{FINAL}
\setboolean{FINAL}{false}
\newcommand{\added}[1]{%
  \ifthenelse{\boolean{FINAL}}%
  {#1}%
  {{\color{blue} #1}}%
}
\newcommand{\update}[1]{%
  \ifthenelse{\boolean{FINAL}}%
  {#1}%
  {{\color{red} #1}}%
}

\newcounter{osCounter}
\newcommandx{\oren}[2][1=]{\todo[linecolor=red,backgroundcolor=red!25,bordercolor=red,#1]{OS [\arabic{osCounter}]:  #2}
\stepcounter{osCounter}\xspace}

\begin{document}

\title{	Asymptotically near-optimal RRT \\
				for fast, high-quality, motion planning
			}
\author
{
	\IEEEauthorblockN{Oren Salzman
										and
										Dan Halperin}
	\IEEEauthorblockA{
			Blavatnik School of Computer Science,
			Tel-Aviv University, Israel}
		\thanks{This work has been supported in part by the 7th 
						Framework Programme for Research of the European Commission, under
						FET-Open grant number 255827 (CGL---Computational Geometry 
						Learning), by the Israel Science Foundation (grant no. 1102/11), 
						by the German-Israeli Foundation (grant no. 1150-82.6/2011), and 
						by the Hermann Minkowski--Minerva Center for Geometry at Tel Aviv 
						University.}
		\thanks{A preliminary and partial version of this manuscript appeared in 
						the proceedings of the 2014 IEEE International Conference on 
						Robotics and Automation (ICRA 2014), pages 4680-4685.}
}


\maketitle


\begin{abstract}
We present \emph{Lower Bound Tree-RRT} (\eRRTbold), a single-query sampling-based algorithm that is \emph{asymptotically near-optimal}. 
Namely, the solution extracted from \eRRTbold converges to a solution that is within an approximation factor of $\mathbf{1+\boldsymbol{\varepsilon}}$ of the optimal solution.
Our algorithm allows for a continuous interpolation between the fast RRT algorithm and the asymptotically optimal RRT* and RRG algorithms. 
When the approximation factor is $\mathbf{1}$ (i.e., no approximation is allowed), \eRRTbold behaves like RRG.
When the approximation factor is unbounded, \eRRTbold behaves like  RRT.
In between, \eRRTbold is shown to produce paths that have higher quality than RRT would produce and run faster than RRT* would run.
This is done by maintaining a tree which is a sub-graph of the RRG roadmap 
and a second, auxiliary graph, which we call the \emph{lower-bound} graph.
The combination of the two roadmaps, which is faster to maintain than the roadmap maintained by RRT*, efficiently guarantees asymptotic near-optimality.
We suggest to use \eRRTbold for high-quality, anytime motion planning.
We demonstrate the performance of the algorithm for scenarios ranging from 3 to 12 degrees of freedom and show that even for small approximation factors, the algorithm produces high-quality solutions (comparable to RRG and RRT*)  with little running-time overhead when compared to RRT.
\end{abstract}
\IEEEpeerreviewmaketitle

\section{Introduction and related work}
Motion planning is a fundamental research topic in robotics with applications in diverse domains such as surgical planning, computational biology, autonomous exploration,
search-and-rescue, and warehouse management.
Sampling-based planners such as PRM~\cite{KSLO96}, RRT~\cite{KL00} and their many variants enabled solving motion-planning problems that had been previously considered infeasible~\cite[C.7]{CBHKKLT05}. 
Recently, there is growing interest in the robotics community in finding \emph{high-quality} paths, which turns out to be a non-trivial problem~\cite{KF11, NRH10}. 
Quality can be measured in terms of, for example, length, clearance, smoothness, energy, to mention a few criteria, or some combination of the above.

\subsection{High-quality planning with sampling-based algorithms}
Unfortunately, planners such as RRT and PRM produce solutions that may be far from optimal~\cite{KF11, NRH10}. 
Thus, many variants of these algorithms and heuristics were proposed in order to produce high-quality paths.

\vspace{2mm}
\textbf{Post-processing existing paths:}
Post-processing an existing path by applying \emph{shortcutting} is a common, effective, approach to increase path quality; see, e.g.,~\cite{GO07}. 
Typically, two non-consecutive configurations are chosen randomly along the path. If the two configurations can be connected using a straight-line segment in the configuration space and this connection improves the quality of the original path, the segment replaces the original path that connected the two configurations.
The process is continued iteratively until a termination condition~holds.

\vspace{2mm}
\textbf{Path hybridization:}
An inherent problem with path post-processing is that it is local in nature.
A path that was post-processed using shortcutting often remains in the  same homotopy class of the original path.
Carefully combining even a small number of different paths (that may be of low quality) often enables the construction of a higher-quality path~\cite{REH11}.

\vspace{2mm}
\textbf{Online optimization:}
Changing the sampling strategy~\cite{ABDJV98, LTA03, US03, SWT09},
or the connection scheme to a new milestone~\cite{US03, SLN00} are examples of heuristics proposed to create higher-quality solutions.
Additional approaches include, among others, useful cycles~\cite{GO07} and random restarts~\cite{WB08}.

\vspace{2mm}
\textbf{Asymptotically optimal and near-optimal solutions:}
In their seminal work, Karaman and Frazzoli~\cite{KF11} give a rigorous analysis of the performance of the RRT and PRM algorithms. 
They show that with probability one, the algorithms will not produce the optimal path.
By modifying the connection scheme of a new sample to the existing data structure, they propose the PRM* and the RRG and RRT* algorithms
(variants of the PRM and RRT algorithms, respectively)
all of which are shown to be \emph{asymptotically optimal}. 
Namely, as the number of samples tends to infinity, the solution obtained by these algorithms converges to the optimal solution with probability one.
To ensure asymptotic optimality, the number of nodes each new sample is connected~to is proportional to $\log (n)$ (here $n$ is the number of free samples).

As PRM* may produce prohibitively large graphs,
recent work has focused on sparsifying these graphs.
This can be done as a post-processing stage of the PRM*~\cite{SSAH14, MB11-IROS}, 
or as a modification of PRM*~\cite{MB11-ISRR, MB12, DB14}.

The performance of RRT* can be improved using several heuristics that bear resemblance to the lazy approach used in this work~\cite{PKSFTW11}.
Additional heuristics to speed up the convergence rate of RRT* were presented in RRT*-SMART~\cite{INMAH13}.
Recently, RRT$^\#$~\cite{AT13} was suggested as an asymptotically-optimal algorithm with a faster convergence rate when compared to RRT*.
RRT$^\#$ extends its roadmap in a similar fashion to RRT* but adds a replanning procedure.
This procedure  ensures that the tree rooted at the initial state contains lowest-cost path information for vertices which have the potential to be part of the optimal solution.
Thus, in contrast to RRT* which only performs \emph{local} rewiring of the search tree,
RRT$^\#$ efficiently propagates changes to \emph{all} the relevant parts of the roadmap.
Janson and Pavone~\cite{JP13} introduced the asymptotically-optimal Fast Marching Tree algorithm (FMT*). The single-query asymptotically-optimal algorithm maintains a tree as its roadmap.
Similarly to PRM*, FMT* samples $n$ collision-free nodes. It then builds a minimum-cost spanning tree rooted at the initial configuration over this set of nodes
(see Section~\ref{sec:fmt} for further details).
Lazy variants have been proposed both for PRM* and RRG~\cite{LH14} and for FMT*~\cite{SH14-arxiv}.

An alternative approach to improve the running times of these algorithms is to relax the asymptotic optimality to \emph{asymptotic near-optimality}.
An algorithm is said to be asymptotically near-optimal if, given an \emph{approximation~factor}~$\varepsilon$, the solution obtained by the algorithm converges to within a factor of $(1+\varepsilon)$ of the optimal solution with probability one, as the number of samples tends to infinity.
Similar to this work, yet using different methods, Littlefield et al.~\cite{LLB13} recently presented an asymptotic near-optimal variant of RRT* for systems with dynamics.
Their approach however, requires setting different parameters used by their algorithm.

\vspace{2mm}
\textbf{Anytime and online solutions:}
An interesting variant of the basic motion-planning problem is anytime motion-planning:
In this problem, 
the time to plan is not known in advance, and the algorithm may be terminated at any stage.
Clearly, any solution should be found as fast as possible and if time permits, it should be refined to yield a higher-quality solution.

Ferguson and Stentz~\cite{FS06} suggest iteratively running RRT while considering only areas that may potentially improve the existing solution.
Alterovitz et al.~\cite{APD11} suggest the  Rapidly-exploring Roadmap Algorithm (RRM), which finds an initial path similar to  RRT. 
Once such a path is found, RRM either explores further the configuration space or refines the explored space.
Luna et al.~\cite{LSMK13} suggest alternating between path shortcutting and path hybridization in an anytime fashion.

RRT* was also adapted for online motion planning~\cite{KWPFT11}.
Here, an initial path is computed and the robot begins its execution.
While the robot moves along this path, the algorithm refines the part that the robot has not yet moved along.

\subsection{Contribution}
We present \eRRT, a single-query sampling-based algorithm that is \emph{asymptotically near-optimal}. 
Namely, the solution extracted from \eRRT converges to a solution that is within a factor of $(1+\varepsilon)$ of the optimal solution.
\eRRT allows for interpolating between the fast, yet sub-optimal, RRT algorithm and the asymptotically-optimal RRG algorithm. 
By choosing $\varepsilon = 0$ no approximation is allowed and \eRRT maintains a roadmap identical to the one maintained by RRG.
Choosing $\varepsilon = \infty$ allows for any approximation and \eRRT maintains a tree identical to the tree maintained by RRT.

The asymptotic near-optimality of \eRRT is achieved by simultaneously maintaining two roadmaps. 
Both roadmaps are defined over the same set of vertices but each consists of a different set of edges.
On the one hand, a path in the first roadmaps may not be feasible, but its cost is always a \emph{lower bound} on the cost of paths extracted from RRG (using the same sequence of random nodes).
On the other hand, a path extracted from the second roadmap is always feasible and its cost is within a factor of $(1+\varepsilon)$ from the lower bound provided by the first roadmap.

We suggest to use \eRRTbold for high-quality, anytime motion planning.
We demonstrate its performance on scenarios ranging from 3 to~12 degrees of freedom (DoF) and show that the algorithm produces high-quality solutions (comparable to RRG and RRT*)  with little running-time overhead when compared to RRT.

This paper is a modified and extended version of a publication presented at the 
2014 IEEE International Conference on Robotics and Automation~\cite{SH14}.
In this paper we present additional experiments and 
extensions of the original algorithmic framework.
Finally, we note that the conference version of this paper
contained an oversight with regard to the roadmap that is used for the lower bound.
We explain the problem and its fix in detail in Section~\ref{sec:alg} after providing all the necessary technical background.

\subsection{Outline}
In Section~\ref{sec:background} we review the RRT, RRG and RRT* algorithms. In Section~\ref{sec:alg} we present our algorithm \eRRT and 
a proof of its asymptotic near-optimality.  
We continue in Section~\ref{sec:eval} to demonstrate in simulations its favorable characteristics on several scenarios.
In Section~\ref{sec:extensions} we discuss a modification of the framework to further speed up the convergence to high-quality solutions.
We conclude in Section~\ref{sec:con} by describing possible directions for future work.
\textVersion{}{In the appendix we list several applications where either RRT or RRT* were used and argue that \eRRT may serve as a superior alternative with no fundamental modification to the underlying algorithms using RRT or RRT*. 
Moreover, we discuss alternative implementations of \eRRT using tools developed for either RRT or RRT* that can enhance \eRRT.
Finally, we demonstrate how the framework presented in this paper for relaxing the optimality of RRG can be used to have a similar effect on another asymptotically-optimal sampling-based algorithm, FMT*~\cite{JP13}.
}


\section{Terminology and algorithmic background}
\label{sec:background}
We begin this section by formally stating the motion-planning problem and introducing several standard procedures used by sampling-based algorithms.
We continue by reviewing the RRT, RRG and RRT* algorithms.

\subsection{Problem definition and terminology}
We follow the formulation of the motion-planning problem as 
presented by Karaman and Frazzoli~\cite{KF11}.
Let $\calX$ denote the configuration space (\Cs), 
\Cfree and \Cforb denote the free and forbidden spaces, respectively.
Let $(\Cfree, x_{\text{init}}, \calX_{goal})$ be the motion-planning problem where:
$x_{\text{init}} \in \Cfree$ is an initial free configuration and
$\calX_{goal} \subseteq \Cfree$ is the goal region.
A \emph{collision-free path} $\sigma : [0,1] \rightarrow \Cfree $ is a continuous mapping to the free space.
It is \emph{feasible} if $\sigma(0)\!=\!x_{\text{init}}$ and $\sigma(1)\!\in\!X_{goal}$.

We will make use of the following procedures throughout the paper: 
\texttt{sample\_free}, a procedure returning a random free configuration;
\texttt{nearest\_neighbor}$(x,V)$ and \texttt{nearest\_neighbors}$(x,V,k)$ are procedures returning the nearest neighbor and $k$ nearest neighbors of $x$ within the set~$V$, respectively.
Let \texttt{steer}$(x,y)$ return a configuration $z$ that is closer to $y$ than $x$ is,
\texttt{collision\_free}$(x,y)$ tests if the straight line segment connecting $x$ and $y$ is contained in \Cfree and let
\texttt{cost}$(x,y)$  be a procedure returning the cost of the straight-line  path connecting $x$ and $y$.
Let us denote by $\texttt{cost}_{\calG}(x)$ the minimal cost of reaching a node $x$ from $x_{\text{init}}$ using a roadmap \calG.
These are standard procedures used by the RRT or RRT* algorithms.
Finally, we use the (generic) predicate \texttt{construct\_roadmap} to assess if a stopping criterion has been reached to terminate the algorithm\footnote{A stopping criterion can be, for example, reaching a certain number of samples or exceeding a fixed time budget.}.

\subsection{Algorithmic background}

The RRT, RRG and RRT* algorithms share the same high-level structure. 
They maintain a roadmap as the underlying data structure 
which is a directed tree for RRT and RRT* and a directed graph for RRG.
At each iteration a configuration $x_{\text{rand}}$ is sampled at random. 
Then, $x_{\text{nearest}}$, the nearest configuration to $x_{\text{rand}}$ in the roadmap is found and extended in the direction of $x_{\text{rand}}$ to a new configuration~$x_{\text{new}}$.
If the path between $x_{\text{nearest}}$ and $x_{\text{new}}$ is collision-free, $x_{\text{new}}$ is added to the roadmap
~(see Alg.~\ref{alg_RRT_orig},~\ref{alg_RRG} and~\ref{alg_RRT_star}, lines 3-9).

\begin{algorithm}[t,b]
\caption{RRT ($x_{\text{init}}$ )}
\label{alg_RRT_orig}
\begin{algorithmic}[1]
 \STATE $\calT.V \leftarrow \set{x_{\text{init}}}$
 \WHILE {\texttt{construct\_roadmap()}}
 	\vspace{3mm}

 	\STATE $x_{\text{rand}} \leftarrow \texttt{sample\_free()}$
 	\STATE $x_{\text{nearest}} \leftarrow
 											\texttt{nearest\_neighbor}(	x_{\text{rand}}, \calT.V)$
 	\STATE $x_{\text{new}} \leftarrow
 											\texttt{steer}(	x_{\text{nearest}}, x_{\text{rand}})$
 	\vspace{3mm}
 	
 	\IF {(!\texttt{collision\_free}($x_{\text{nearest}}, x_{\text{new}}$))}
 		\STATE CONTINUE
	\ENDIF
 	\vspace{3mm}
 	
 	\STATE	$\calT.V \leftarrow \calT.V \cup \set{x_{\text{new}}}$
 	\STATE	$\calT.\texttt{parent}(x_{\text{new}}) \leftarrow x_{\text{nearest}}$
 \ENDWHILE
 	
\end{algorithmic}
\end{algorithm}

\begin{algorithm}[t,b]
\caption{RRG ($x_{\text{init}}$ )}
\label{alg_RRG}
\begin{algorithmic}[1]
 \STATE $\calG.V \leftarrow \set{x_{\text{init}}}$
 				\hspace{3mm}
 				$\calG.E \leftarrow \emptyset$
 \WHILE {\texttt{construct\_roadmap()}}
 	\vspace{3mm}

 	\STATE $x_{\text{rand}} \leftarrow \texttt{sample\_free()}$
 	\STATE $x_{\text{nearest}} \leftarrow
 											\texttt{nearest\_neighbor}(	x_{\text{rand}}, \calG.V)$
 	\STATE $x_{\text{new}} \leftarrow
 											\texttt{steer}(	x_{\text{nearest}}, x_{\text{rand}})$
 	\vspace{3mm}
 	
 	\IF {(!\texttt{collision\_free}($x_{\text{nearest}}, x_{\text{new}}$))}
 		\STATE CONTINUE
	\ENDIF
 	\vspace{3mm}
 	
 	\STATE	$\calG.V \leftarrow \calG.V \cup \set{x_{\text{new}}}$
 	\STATE	$\calG.E \leftarrow 
 											\set{(x_{\text{nearest}}, x_{\text{new}}) , (x_{\text{new}}, x_{\text{nearest}})}$
  \vspace{3mm}

  \STATE $X_{\text{near}} \leftarrow \texttt{nearest\_neighbors}(	x_{\text{new}},$\\ 
  			\hspace{35mm}
  			$\calG.V , k_{RRG} \log(|\calG.V |))$
	\FORALL {$(x_{\text{near}}, X_{\text{near}})$}
		\IF {(\texttt{collision\_free}($x_{\text{near}}, x_{\text{new}}$))}
 			\STATE	$\calG.E \leftarrow 
 													\set{(x_{\text{near}}, x_{\text{new}}) , (x_{\text{new}}, x_{\text{near}})}$
		\ENDIF
  \ENDFOR
 \ENDWHILE
 	
\end{algorithmic}
\end{algorithm}

\begin{algorithm}[t,b]
\caption{RRT* ($x_{\text{init}}$ )}
\label{alg_RRT_star}
\begin{algorithmic}[1]
 \STATE $\calT.V \leftarrow \set{x_{\text{init}}}$
 \WHILE {\texttt{construct\_roadmap()}}
 	\vspace{3mm}

 	\STATE $x_{\text{rand}} \leftarrow \texttt{sample\_free()}$
 	\STATE $x_{\text{nearest}} \leftarrow
 											\texttt{nearest\_neighbor}(	x_{\text{rand}}, \calG.V)$
 	\STATE $x_{\text{new}} \leftarrow
 											\texttt{steer}(	x_{nearest}, x_{rand})$
 	\vspace{3mm}
 	
 	\IF {(!\texttt{collision\_free}($x_{\text{nearest}}, x_{\text{new}}$))}
 		\STATE CONTINUE
	\ENDIF
 	\vspace{3mm}
 	
 	\STATE	$\calT.V \leftarrow \calT.V \cup \set{x_{\text{new}}}$
 	\STATE	$\calT.\texttt{parent}(x_{\text{new}}) \leftarrow x_{\text{nearest}}$
  \vspace{3mm}

  \STATE $X_{\text{near}} \leftarrow \texttt{nearest\_neighbors}(	x_{\text{new}},$\\ 
  			\hspace{35mm}
  			$\calT.V , k_{RRG} \log(|\calT.V |))$

	\vspace{3mm}
  
	\FORALL {$(x_{\text{near}}, X_{\text{near}})$}
  	\STATE  \texttt{rewire\_RRT$^*$}($x_{\text{near}}, x_{\text{new}}$ )
  \ENDFOR

	\vspace{3mm}
  
	\FORALL {$(x_{\text{near}}, X_{\text{near}})$}
  	\STATE  \texttt{rewire\_RRT$^*$}($x_{\text{new}}, x_{\text{near}}$ )
  \ENDFOR
 \ENDWHILE
 	
\end{algorithmic}
\end{algorithm}

\begin{algorithm}[h,t, b]
\caption{\texttt{rewire\_RRT$^*$}($x_{\text{potential\_parent}}, x_{\text{child}}$)}
\label{alg_update_rrt_star}
\begin{algorithmic}[1]
  
  \IF {(\texttt{collision\_free}($x_{\text{potential\_parent}}, x_{\text{child}}$))}
  	\STATE	$c \leftarrow$ \texttt{cost}($x_{\text{potential\_parent}}, x_{\text{child}}$)
  	\IF {	($\texttt{\text{cost}}_{\calT}(x_{\text{potential\_parent}}) + c
  						<
  				  \texttt{cost}_{\calT}(x_{\text{child}}))$}      
	    \STATE	$\calT.\texttt{\text{parent}}(x_{\text{child}}) \leftarrow 
    						x_{\text{potential\_parent}}$
    \ENDIF
 	\ENDIF

\end{algorithmic}
\end{algorithm}						

The algorithms differ in the connections added to the roadmap.
In RRT, only the edge $(x_{\text{nearest}}, x_{\text{new}})$ is added.
In RRG and RRT*, a set $X_{\text{near}}$ of $k_{RRG} \log(|V|)$ nearest neighbors of $x_{\text{new}}$ is considered. 
Here, $k_{RRG}$ is a constant ensuring that the cost of paths produced by  RRG and RRT* indeed converges to the optimal cost almost surely as the number of samples grows. 
A valid choice for all problem instances is $k_{RRG} = 2e$~\cite{KF11}. 
For each neighbor $x_{\text{near}} \in X_{\text{near}}$ of $x_{\text{new}}$, RRG  checks if the path between $x_{\text{near}}$ and $x_{\text{new}}$ is collision-free and if so, $(x_{\text{near}} , x_{\text{new}})$ and $(x_{\text{new}}, x_{\text{near}})$ are added to the roadmap (lines 10-13).
RRT* maintains a sub-graph of the RRG roadmap.
This is done by an additional rewiring procedure 
(Alg.~\ref{alg_update_rrt_star})
which is invoked twice: 
The first time, it is used to find the node $x_{\text{near}} \in X_{\text{near}}$ which will minimize the cost to reach $x_{\text{new}}$ 
(Alg.~\ref{alg_RRT_star}, lines 11-12).
The second time, the procedure is used to  to minimize the cost to reach every node $x_{\text{near}} \in X_{\text{near}}$ by considering $x_{\text{new}}$ as its parent (Alg.~\ref{alg_RRT_star}, lines~13-14).
Thus, at all time, RRT* maintains a tree which, as mentioned, is a subgraph of the RRG roadmap.

Given a sequence of $n$ random samples, 
the cost of the path obtained using the RRG algorithm is a lower bound on the cost of the path obtained using the RRT* algorithm.
However, RRG requires both 
additional memory (to explicitly store the set of $O(\log n)$ neighbours)
and
exhibits longer running times (due to the additional calls to the local planner).
In practice, this excess in running time is far from negligible (see Section~\ref{sec:eval}), making RRT* a more suitable algorithm for asymptotically-optimal motion planning. 
\section{Asymptotically near-optimal motion-planning}
\label{sec:alg}
Clearly the asymptotic optimality of the RRT* and RRG algorithms comes at the cost of the additional $O(k_{RRG}\log(|V|))$ calls to the local planner at each stage (and some additional overhead).
If we are not concerned with \emph{asymptotically optimal} solutions, we do not have to consider all of the $k_{RRG} \log(|V|)$ neighbors when a node is added. 
Our idea is to initially only \emph{estimate} the quality of each edge.
We use this estimate of the quality of the edge to decide if to 
discard it, 
use it \emph{without} checking if it is collision-free or 
use it after validating that it is indeed collision-free.
Thus, many calls to the local planner can be avoided, though we still need to estimate the quality of many edges. 
Our approach is viable in cases where such an assessment can be carried out efficiently. 
Namely, more efficiently than deciding if an edge is collision-free.
This condition holds naturally when the quality measure is \emph{path length}
which is the cost function considered in this paper; for a discussion on different cost functions, see Section~\ref{sec:con}.

\subsection{Single-sink shortest-path problem}

As we will see, our algorithm needs to maintain the shortest path from 
$x_{\text{init}}$ to any node in a graph.
Moreover, this graph undergoes a series of edge insertions and edge deletions.
This problem is referred to as the fully dynamic \emph{single-source shortest-path problem} or SSSP for short. 
Efficient algorithms~\cite{FMN00, RR96} exist that can store the minimal cost to reach each node (and the corresponding path) in such settings from a source node.
In our setting, this source node is $x_{\text{init}}$.
We make use of the following procedures which are provided by SSSP algorithms:
\texttt{delete\_edge$_{\text{SSSP}}$}($\calG, (x_1, x_2)$)
and
\texttt{insert\_edge$_{\text{SSSP}}$}($\calG, (x_1, x_2)$) 
which 
delete and 
insert, respectively, the edge $(x_1, x_2)$ from/into the graph $\calG$
while maintaining cost$_\calG$ for each node.
We assume that these procedures return the set of nodes whose cost has changed due to the edge deletion or edge insertion.
Furthermore, let \texttt{parent}$_{\texttt{SSSP}}$($\calG, x$) be a procedure returning the parent of $x$ in the shortest path from the source to~$x$ in $\calG$.

\subsection{\eRRT}
We propose a modification to the RRG algorithm by maintaining two roadmaps $\Glb, \Tub$ simultaneously.
Both roadmaps have the same set of vertices but differ in their edge set.
$\Glb$ is a graph and $\Tub$ is a tree rooted at $x_{\text{init}}$\footnote{The subscript of \Glb is an abbreviation for lower bound and the subscript of \Tub is an abbreviation for approximation.}.

Let $\calG_{RRG}$ be the roadmap constructed by RRG if run on the same sequence of samples used for \eRRT.
The following invariants are maintained by the \eRRT algorithm:\\
\begin{framed}
\textbf{Bounded approximation invariant} - 
For every node $x \in \Tub, \Glb$,
$
	\texttt{cost}_{\Tub}(x) 
			\leq 
	(1+\varepsilon) \cdot \texttt{cost}_{\Glb}(x).
$
\end{framed}
and
\begin{framed}
\textbf{Lower bound invariant} - 
For every node $x \in \Glb$,
$
	\texttt{cost}_{\Glb}(x) \leq \texttt{cost}_{\calG_{RRG}}(x).
$
\end{framed}
The lower bound invariant is maintained by ensuring that the edges of $\calG_{RRG}$ are a subset of the edges of $\Glb$.
As we will see, $\Glb$ may possibly contain some edges that $\calG_{RRG}$ considered but found to be in collision.

The main body of the algorithm (see Alg.~\ref{alg_RRT}) follows the structure of the RRT, RRT* and RRG algorithms with respect to adding a new milestone (lines 3-7) but differs in the connections added.
If a path between the new node $x_{\text{new}}$ and its nearest neighbor $x_{\text{nearest}}$ is indeed collision-free, it is added to both roadmaps together with an  edge from $x_{\text{nearest}}$ to $x_{\text{new}}$ (lines~8-11).

Similar to RRG and RRT*, \eRRT locates the set $X_{\text{near}}$ of $k_{RRG}\log(|V|)$ nearest neighbors of $x_{\text{new}}$ (line~12).
Then, 
for each edge connecting a node from $X_{\text{near}}$ to $x_{\text{new}}$  and
for each edge connecting $x_{\text{new}}$ to a node from $X_{\text{near}}$,  it uses a procedure 
\texttt{consider\_edge}
(Alg.~\ref{alg_update}) to assess if the edge should be inserted to either roadmaps. 
The edge is first lazily inserted into \Glb without checking if it is collision-free.
This \emph{may} cause the bounded approximation invariant to be violated, which in turn will induce a call to the local planner for a set of edges.
Each such edge might either be inserted into \Tub or removed from \Glb.

This is done as follows,
first, the edge considered is inserted to $\Glb$ while updating the shortest path to reach each vertex in  $\Glb$ (Alg.~\ref{alg_update}, line~1).
Denote by $I$ the set of updated vertices after the edge insertion.
Namely, for every $x \in I$, \texttt{cost}$_{\Glb}(x)$
has decreased due to the edge insertion.
This cost decrease may, in turn, cause the bounded approximation invariant to be violated for some nodes in~$U$.
All such nodes are collected and inserted into a priority queue~$Q$ (line 2)
ordered according to $\texttt{cost}_{\Glb}$ from low to high.
Now, the algorithm proceeds in iterations until the queue is empty (lines~3-15).
At each iteration, the head of the queue $x$ is considered (line~4). 
If the bounded approximation invariant does not hold (line~5), 
the algorithm checks if the edge in \Glb connecting the node~$x$ to its parent along the shortest path to $x_{\text{init}}$ is collision free (lines~6-7).
If this is the case, 
the approximation tree is updated (line~8) and the head of the queue is removed (line~9).
If not, 
the edge is removed from \Glb (line~11).
This causes an increase in cost$_{\text{\Glb}}$ for a set $D$ of nodes, 
some of which are already in the priority queue.
Clearly, the bounded approximation invariant holds for the nodes $x \in D$ that are not in the priority queue. 
Thus, we take only the nodes $x \in D$ that are already in $Q$ and update their location in~$Q$ according to their new cost (lines~12-13) .
Finally,  if the bounded approximation invariant holds for $x$ then it is removed from the queue (lines~15).


\begin{algorithm}[t,b]
\caption{\eRRT ($x_{\text{init}}, \varepsilon$ )}
\label{alg_RRT}
\begin{algorithmic}[1]
 \STATE $\Tlb.G \leftarrow \set{x_{\text{init}}}$
 				\hspace{2mm}
 		$\Tub.V \leftarrow \set{x_{\text{init}}}$
 \WHILE {\texttt{construct\_roadmap()}}
 	\textVersion{\vspace{2mm}}{\vspace{3mm}}

 	\STATE $x_{\text{rand}} \leftarrow \texttt{sample\_free()}$
 	\STATE $x_{\text{nearest}} \leftarrow
 											\texttt{nearest\_neighbor}(	x_{rand}, \Tlb.V)$
 	\STATE $x_{\text{new}} \leftarrow
 											\texttt{steer}(	x_{\text{nearest}}, x_{\text{rand}})$
 	\textVersion{\vspace{2mm}}{\vspace{3mm}}
 	
 	\IF {(!\texttt{collision\_free}($x_{\text{nearest}}, x_{\text{new}}$))}
 		\STATE CONTINUE
	\ENDIF
 	\textVersion{\vspace{2mm}}{\vspace{3mm}}

 	\STATE	$\Tub.V \leftarrow \Tub.V \cup \set{x_{\text{new}}}$ 	
 	\STATE	$\Tub.\texttt{parent}(x_{\text{new}}) \leftarrow x_{\text{nearest}}$
 	
 	\textVersion{\vspace{2mm}}{\vspace{3mm}}
 	
 	\STATE	$\Glb.V \leftarrow \Glb.V \cup \set{x_{\text{new}}}$
	\STATE 	\texttt{insert\_edge$_{\text{SSSP}}$}($\Glb, (x_{\text{nearest}}, x_{\text{new}})$)

 	\textVersion{\vspace{2mm}}{\vspace{3mm}}

	\STATE $X_{\text{near}} \leftarrow \texttt{nearest\_neighbors}(	x_{\text{new}},$\\ 
  			\hspace{35mm}
  			$\Glb.V , k_{RRG} \log(|\Glb.V |))$
  			
 	\textVersion{\vspace{2mm}}{\vspace{3mm}}

	\FORALL {$(x_{\text{near}}, X_{\text{near}})$}
		\STATE \texttt{consider\_edge}$(x_{\text{near}}, x_{\text{new}})$
	\ENDFOR

 	\textVersion{\vspace{2mm}}{\vspace{3mm}}

	\FORALL {$(x_{\text{near}}, X_{\text{near}})$}
		\STATE \texttt{consider\_edge}$(x_{\text{new}}, x_{\text{near}})$
	\ENDFOR

 \ENDWHILE
 	
\end{algorithmic}
\end{algorithm}

\begin{algorithm}[t,b]
\caption{\texttt{consider\_edge}($x_{1}, x_{2}$)}
\label{alg_update}
\begin{algorithmic}[1]
	\STATE $I \leftarrow $\texttt{insert\_edge$_{\text{SSSP}}$}($\Glb, (x_{1}, x_{2})$)
	\STATE $Q \leftarrow \set{ x \in I \ | \
 		\text{cost}_{\Tub}(x) > (1 + \varepsilon) \cdot \text{cost}_{\Glb}(x) }$
	\WHILE {$Q \neq \emptyset$}
		\STATE $x \leftarrow Q.\texttt{top}()$; 
		\IF {$\text{cost}_{\Tub}(x) > (1 + \varepsilon) \cdot \text{cost}_{\Glb}(x)$}
			\STATE $x_{parent} \leftarrow \texttt{parent}_{\texttt{SSSP}}(\Glb, x)$	

	 		\IF {(\texttt{collision\_free}
 				($x_{parent}, x$))}
				\STATE	$\Tub.\texttt{\text{parent}}(x) \leftarrow  x_{\text{parent}}$  
	   			\STATE	$Q.\texttt{pop}()$
   	       	\ELSE
   				\STATE $D \leftarrow $\texttt{delete\_edge$_{\text{SSSP}}$}($\Glb, (x_{\text{parent}}, x)$)
				\FORALL {$y \in D \cap Q $}
					\STATE $Q$.\texttt{update\_cost}$(y)$
				\ENDFOR
   			\ENDIF
   		\ELSE
   			\STATE	$Q.\texttt{pop}()$
       	\ENDIF
	\ENDWHILE
\end{algorithmic}
\end{algorithm}

\subsection{Analysis}
\label{susbsec:analysis}
In this section we show that Alg.~\ref{alg_RRT} maintains the lower bound invariant (Corollary~\ref{cor_lb}) and 
that after every iteration of the algorithm
the bounded approximation invariant is maintained (Lemma~\ref{lem_invariant}).
We then report on the time complexity of the algorithm (Corollary~\ref{cor_complex}).

We note the following straightforward, yet helpful observations comparing \eRRT and RRG when run on the same sequence of random samples:
\begin{obs}
\label{obs:1}
A node $x$ is added to $\Glb$ and to $\Tub$ if and only if $x$ is added to $\calG_{RRG}$ 
(Alg.~\ref{alg_RRG} lines 3-8 and~\ref{alg_RRT}, lines~3-11). 

\end{obs}
\begin{obs}
Both \eRRT and RRG consider the same set of $k_{RRG}\log(|V|)$ nearest neighbors of~$x_{\text{new}}$~(Alg.~\ref{alg_RRG}, line~10 and Alg.~\ref{alg_RRT}, line~12).

\end{obs}
\begin{obs}
Every edge added to the RRG roadmap 
(Alg.~\ref{alg_RRG} line 13)
is added to $\Glb$
(Alg.~\ref{alg_RRT} lines 14, 16 and Alg.~\ref{alg_update} line 1).
\end{obs}
Note that some additional edges may be added to $\Glb$ which are not added to the RRG roadmap as they are not collision-free.

\begin{obs}
\label{obs:4}
Every edge of $\Tub$ is collision free~(Alg.~\ref{alg_RRT}, line 9 and Alg.~\ref{alg_update}, line~8).
\end{obs}

Thus, the following corollary trivially holds:
\begin{cor}
\label{cor_lb}
	After every iteration of \eRRT (Alg.~\ref{alg_RRT}, lines 3-16) the lower 
	bound invariant is maintained.
\end{cor}

We continue with the following observations relevant to the analysis of the procedure 
\texttt{consider\_edge}($x_{1}, x_{2}$):
\begin{obs}
\label{obs:5}
The only place where cost$_{\Glb}$ is decreased is during a call to
\texttt{insert\_edge$_{\text{SSSP}}$}($\Glb, (x_{1}, x_{2})$ (Alg.~\ref{alg_update}, line~1).
\end{obs}

\begin{obs}
\label{obs:7}
A node~$x$ is removed from the queue~$Q$ (Alg~\ref{alg_update}, lines~9,15) only if
the bounded approximation invariant holds for $x$.
\end{obs}

Showing that the bounded approximation invariant is maintained is done by 
induction on the number of calls to \texttt{consider\_edge}($x_{1}, x_{2}$).
Using Obs.~\ref{obs:5}, prior to the first call to  \texttt{consider\_edge}($x_{1}, x_{2}$) the bounded approximation invariant is maintained.
Thus, we need to show that:
\begin{lem}
\label{lem_invariant}
	If the bounded approximation invariant holds prior to a call to the procedure \texttt{consider\_edge}($x_{1}, x_{2}$) (Alg.~\ref{alg_update}), then the procedure will terminate with the invariant maintained.
\end{lem}

\begin{proof}
Assume that the bounded approximation invariant was maintained prior to a call to \texttt{consider\_edge}($x_{1}, x_{2}$).
By Observation~\ref{obs:5} inserting a new edge (line~1) may cause the bounded approximation invariant to be violated for a set of nodes. Moreover, it is the \emph{only} place where such an event can occur.
Observation~\ref{obs:7} implies that the bounded approximation invariant holds for every vertex \emph{not} in~$Q$. 

Recall that in the priority queue  we order the nodes according to $\texttt{cost}_{\Glb}$ (from low to high) and at each iteration of \texttt{consider\_edge}($x_{1}, x_{2}$)   the top of the priority queue~$x$ is considered.
The parent $x_{\text{parent}}$ of $x$, that has a smaller cost value, cannot be in the priority queue.
Thus, the bounded approximation invariant holds for $x_{\text{parent}}$.
Namely,
\[
	\texttt{cost}_{\Tub}(x_{\text{parent}}) 
			\leq 
	(1+\varepsilon) \cdot \texttt{cost}_{\Glb}(x_{\text{parent}}).
\]
Now, if the edge between $x_{\text{parent}}$ and $x$ is found to be free (line~7), we update the approximation tree (line~8).
It follows that after such an event, 
\begin{eqnarray}
\texttt{cost}_{\Tub}(x)
		&		= 	&		
\texttt{cost}_{\Tub}(x_{parent}) + 
 		\nonumber \\
		&			&		\texttt{cost}(x_{\text{parent}}, x)
 		\nonumber \\
		&		\leq 	&		
(1+\varepsilon) \cdot \texttt{cost}_{\Glb}(x_\text{parent}) + 
 		\nonumber \\
		&			&		\texttt{cost}(x_\text{parent},x)
 		\nonumber \\
		&		\leq 	&		
(1+\varepsilon) \cdot  \texttt{cost}_{\Glb}(x). \nonumber 
\end{eqnarray}
Namely, after updating the approximation tree, the bounded approximation invariant holds for the node~$x$.

To summarize, at each iteration of Alg.~\ref{alg_update} (lines 3-16), 
either:
(i)~we remove a node~$x$ from~$Q$ (line~9 or line~15)
or
(ii)~we remove an incoming edge to the node $x$ from the lower bound graph (line~11).
If the node~$x$ was removed from~$Q$ (case~(i)), 
the bounded approximation invariant holds---either it was not violated to begin with (line~15)
or it holds after updating  the approximation tree (lines~8-9).


To finish the proof we need to show that the main loop (lines~3-15) in Alg.~\ref{alg_update} indeed terminates.
Recall that the degree of each node is $O(\log n)$.
Thus, a node~$x$ cannot be at the head of the queue more than $O(\log n)$ times (after each time we either remove an incoming edge or remove~$x$ from the queue).
This in turn implies that after at most $O(n\log n)$ iterations~$Q$ is empty and the main loop terminates. 
\end{proof}

From 
Corollary~\ref{cor_lb},
Lemma~\ref{lem_invariant} and
using the asymptotic optimality of RRG we conclude, 
\begin{thm}
 	\eRRT is asymptotically near-optimal with an approximation factor of $(1+\varepsilon)$. 
\end{thm}

\noindent
Namely, the cost of the path computed by \eRRT converges to a cost at most $(1+\varepsilon)$ times the cost of the optimal path almost surely.	

We continue now to discuss the time complexity of the algorithm.
If $\delta$ is the number of nodes updated during a call to 
an SSSP procedure\footnote{The number of nodes~$\delta$ updated during an SSSP procedure depends on the topology of the graph and the edge weights. 
Theoretically, in the worst case $\delta = O(n)$ and a dynamic SSSP algorithm cannot perform better than recomputing shortest paths from scratch.
However, in practice this value is much smaller. } 
(namely, 
\texttt{insert\_edge}$_{\text{SSSP}}$ 
or
\texttt{delete\_edge}$_{\text{SSSP}}$),
then the complexity of the procedure is 
$O(\delta \log n)$ 
when using the algorithm of Ramalingam et al.~\cite{RR96}.
Set $\hat{\delta}$ to be 
the maximum value of $\delta$ 
over all calls to SSSP procedures
(Alg~\ref{alg_RRT} line~11 and 
Alg~\ref{alg_update}, lines~1 and~11)
and let~$n$ denote the final number of samples used by \eRRT.

We have $O(n\log n)$ edges
and each edge will be inserted to~\Glb once 
(Alg~\ref{alg_RRT} line~11 or 
Alg~\ref{alg_update} line~1)
and possibly be removed from~\Glb once  
(Alg~\ref{alg_update} line~11).
Therefor, the total complexity due to the  SSSP procedures
is $O(\hat{\delta} \cdot n\log^2 n)$.
The time-complexity of all the other operations (nearest neighbours, collision detection etc.) is similar to RRG which runs in time $O( n \log n)$.

\begin{cor}
\label{cor_complex}
	\eRRT runs in time $O(\hat{\delta} \cdot n\log^2 n)$, 
	where~$n$ is the number of samples and
	$\hat{\delta}$ is the maximal number of nodes updated over all 
	SSSP procedures .
\end{cor}

While this running time may seem discouraging, 
we note that in practice, the local planning dominates the actual running time of the algorithm in practice.
As we demonstrate in Section~\ref{sec:eval} through various simulations, 
\eRRT produces high-quality results in an efficient manner.

\subsection{Implementation details}
We describe the following optimizations that we use 
in order to speed up the running-time of the algorithm.
The first is that the set $X_{\text{near}}$ is ordered according to the cost to reach $x_{\text{new}}$ from $x_{\text{init}}$ through an element $x$ of $X_{\text{near}}$.
Hence, the set $X_{\text{near}}$ will be traversed from the node that yields the smallest lower bound to reach $x_{\text{new}}$ to the node that will yield the highest lower bound.
After the first edge that does \emph{not} violate the bounded approximation invariant, no subsequent node can improve the cost to reach $x_{\text{new}}$ and \texttt{insert\_edge}$_{\texttt{SSSP}}$ will not need to perform any updates.
This ordering was previously used  to speed up~RRT* (see, e.g.,~\cite{PKSFTW11, WvdB13}). 

The second optimization comes  to avoid the situation where \texttt{insert\_edge}$_{\texttt{SSSP}}$ is called and immediately afterwards the same edge is removed.
Hence, given an edge, we first check if the bounded approximation invariant will be violated had  the edge been inserted.
If this is indeed the case, the local planner is invoked and only if the edge is collision free \texttt{insert\_edge}$_{\texttt{SSSP}}$ is called.

\subsection{Discussion}
Let ${\rm T}_{ALG}^{\omega}$ denote the time needed for an algorithm $ALG$ to find a feasible solution on a sequence~$\omega$of random samples. 
Clearly, ${\rm T}_{\rm RRT}^{\omega} \leq {\rm T}_{\rm RRG}^{\omega}$ (as RRG may require more calls to the collision detector than the RRT algorithm).
Moreover, for every $\varepsilon_1 \leq \varepsilon_2$ it holds that 
\[
{\rm T}_{\rm RRT}^{\omega} 
	\leq 
{\rm T}_{\rm \eRRT(\varepsilon_2)}^{\omega} 
	\leq
{\rm T}_{\rm \eRRT(\varepsilon_1)}^{\omega} 
	\leq 
{\rm T}_{\rm RRG}^{\omega}.
\]

Thus, given a limited amount of time, RRG may fail to construct any solution. 
On the other hand, RRT may find a solution fast but will not improve its quality (if the goal is a single configuration).
\eRRT allows to find a feasible path quickly while 
continuing to search for a path of higher quality.

\vspace{2mm}
\noindent
\textbf{Remark}
The conference version of this paper contained an oversight with regard to how the bounded approximation invariant was maintained.
Specifically, 
instead of storing~\Glb as a graph,
a tree was stored which was rewired locally.
When the algorithm tested if the bounded approximation invariant was violated for a node~$x$,
it only considered the \emph{children} of~$x$ in the tree.
This local test did not take into account the fact that changing the cost of~$x$ in the tree could also change the cost of nodes~$y$ that are descendants of~$x$ (but not its children).
The implications of the oversight is that the algorithm was not asymptotically near optimal.
The experimental results presented in the conference version of this paper suggest that in certain scenarios this oversight did not have a significant effect on the convergence to high quality solutions.
Having said that, 
\eRRT as presented in this paper is both 
asymptotically near optimal
and 
converges to high quality solutions faster than the original algorithm.

\section{Evaluation}
\label{sec:eval}
We present an experimental evaluation of the performance of \eRRT as an anytime algorithm on different scenarios consisting of 3,6 and 12 DoFs (Fig.~\ref{fig:scenarios}).
The algorithm was implemented using the Open Motion Planning Library (OMPL~0.10.2)~\cite{SMK12} and our implementation is currently distributed with the OMPL release. All experiments were run on a 2.8GHz Intel Core i7 processor with 8GB of memory.
RRT* was implemented by using the ordering optimization described in Section~\ref{sec:alg} and~\cite{PKSFTW11}).

\begin{figure*}[t,b,h]
  \centering
  \subfloat
   [\sf Maze scenario]
   { 
   	\includegraphics[height =3.5 cm]{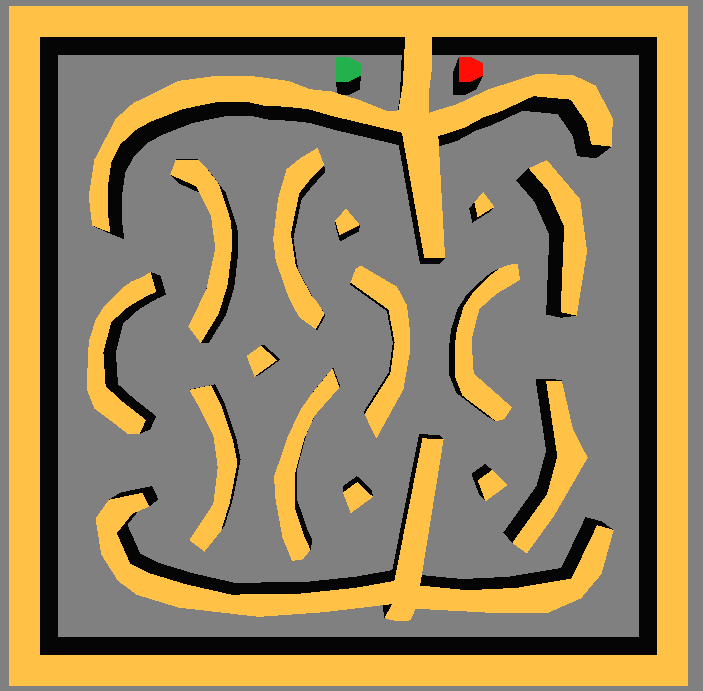}
   	\label{fig:maze}
   }
   \hspace{3mm}
  \subfloat
   [\sf Alternating barriers scenario]
   {
   	\includegraphics[height =3.5 cm]{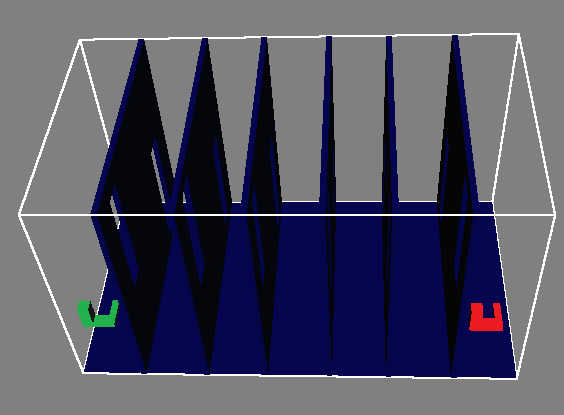}
   	\label{fig:alternating}
   }
   \hspace{3mm}
  \subfloat
   [\sf Cubicles scenario (2-robots)]
   {
   	\includegraphics[height =3.5 cm]{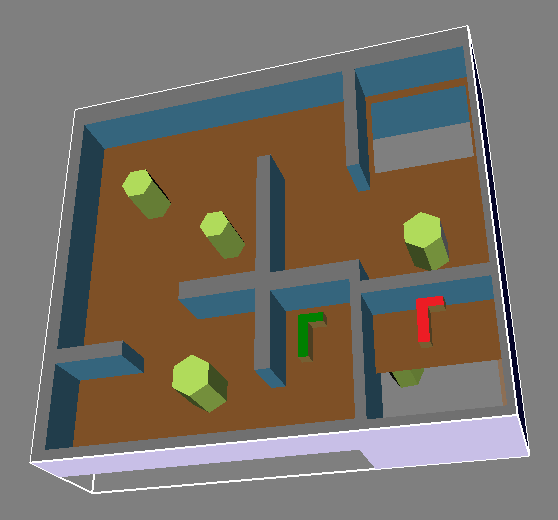}
   	\label{fig:cubicles}
   }
  \caption{\sf 	Benchmark scenarios. 
  							The start and goal configuration are depicted in green and 
  							red, respectively.}
  \label{fig:scenarios}
\end{figure*}

\begin{wrapfigure}{r}{0.18\textwidth}
	\vspace{-2mm}
  \begin{center}
    \includegraphics[height =1.8 cm]{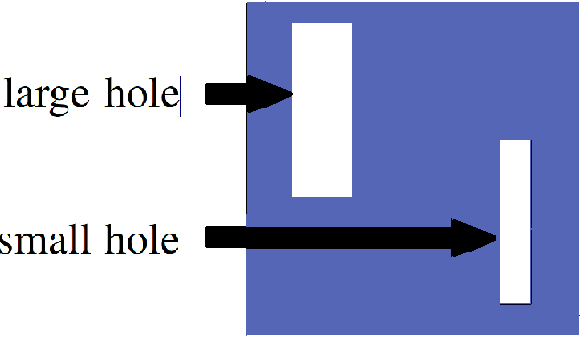}
  \end{center}
  \vspace{-2mm}
  \caption{\sf 	One barrier of the Alternating barriers scenario.}
  \label{fig:barrier}
  \vspace{-3mm}
\end{wrapfigure}

The Maze scenario (Fig.~\ref{fig:maze}) consists of a planar polygonal robot that can translate and rotate.
The Alternating barriers scenario (Fig.~\ref{fig:alternating}) consists of a robot with three perpendicular rods free-flying in space.
The robot needs to pass through a series of barriers each containing a large and a small hole. For an illustration of one such barrier, see Fig.~\ref{fig:barrier}.
The large holes are located at alternating sides of consecutive barriers. 
Thus, an easy path to find would be to cross each barrier through a large hole. 
A high-quality path would require passing through a small hole after each large hole.
Finally, the cubicles scenario consists of two L-shaped robots free-flying in space that need to exchange locations amidst a sparse collection of obstacles\footnote{The Maze Scenario and the Cubicles Scenario are provided as part of the OMPL distribution.}.

We compare the performance of \eRRT with RRT, RRG and RRT* 
when a fixed time budget is given. 
We add another algorithm which we call RRT+RRT* which initially runs RRT and once a solution is found runs RRT*. RRT+RRT* will find a solution as fast as RRT and is asymptotically-optimal.
For \eRRT we consider $(1+\varepsilon)$ values of $1.2, 1.4, 1.8$ and report on the success rate of each algorithm (Fig.~\ref{fig:suc}). Additionally, we report on the path length after applying shortcuts (Fig.~\ref{fig:len}). Each result is averaged over 100 different runs.

\begin{figure*}[t,b,h]
  \centering
  \subfloat
   [\sf Maze Scenario]
   {
	\includegraphics[width=0.25\textwidth]{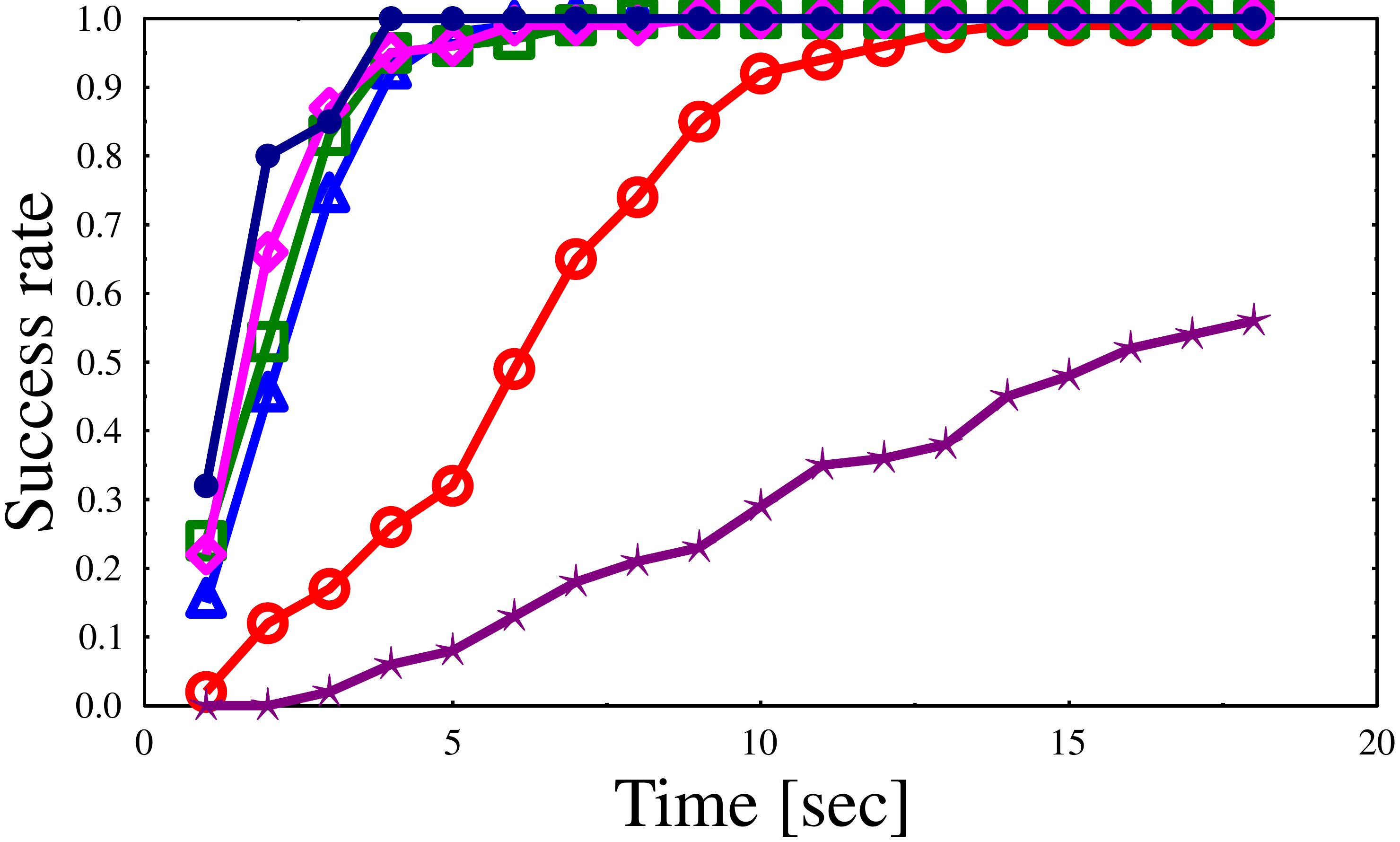}
	\label{fig:maze_suc}
   }
  \hspace{1mm}
  \subfloat
   [\sf Alternating barriers scenario]
   { 
   	\includegraphics[width=0.25\textwidth]
			{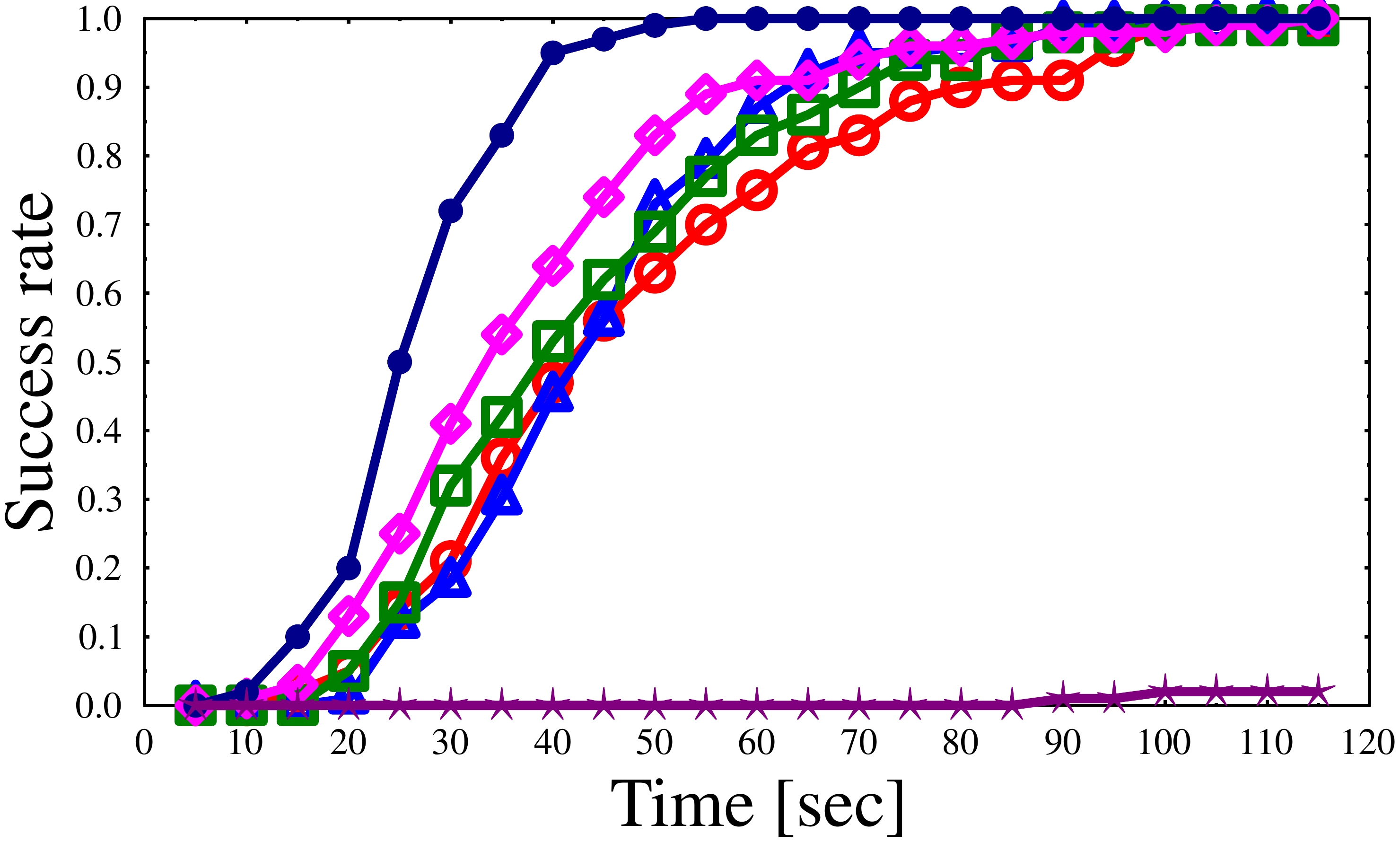}
   	\label{fig:alternating_suc}
   }
   \hspace{1mm}
  \subfloat
   [\sf Cubicles scenario]
   { 
   	\includegraphics[width=0.25\textwidth]
   		{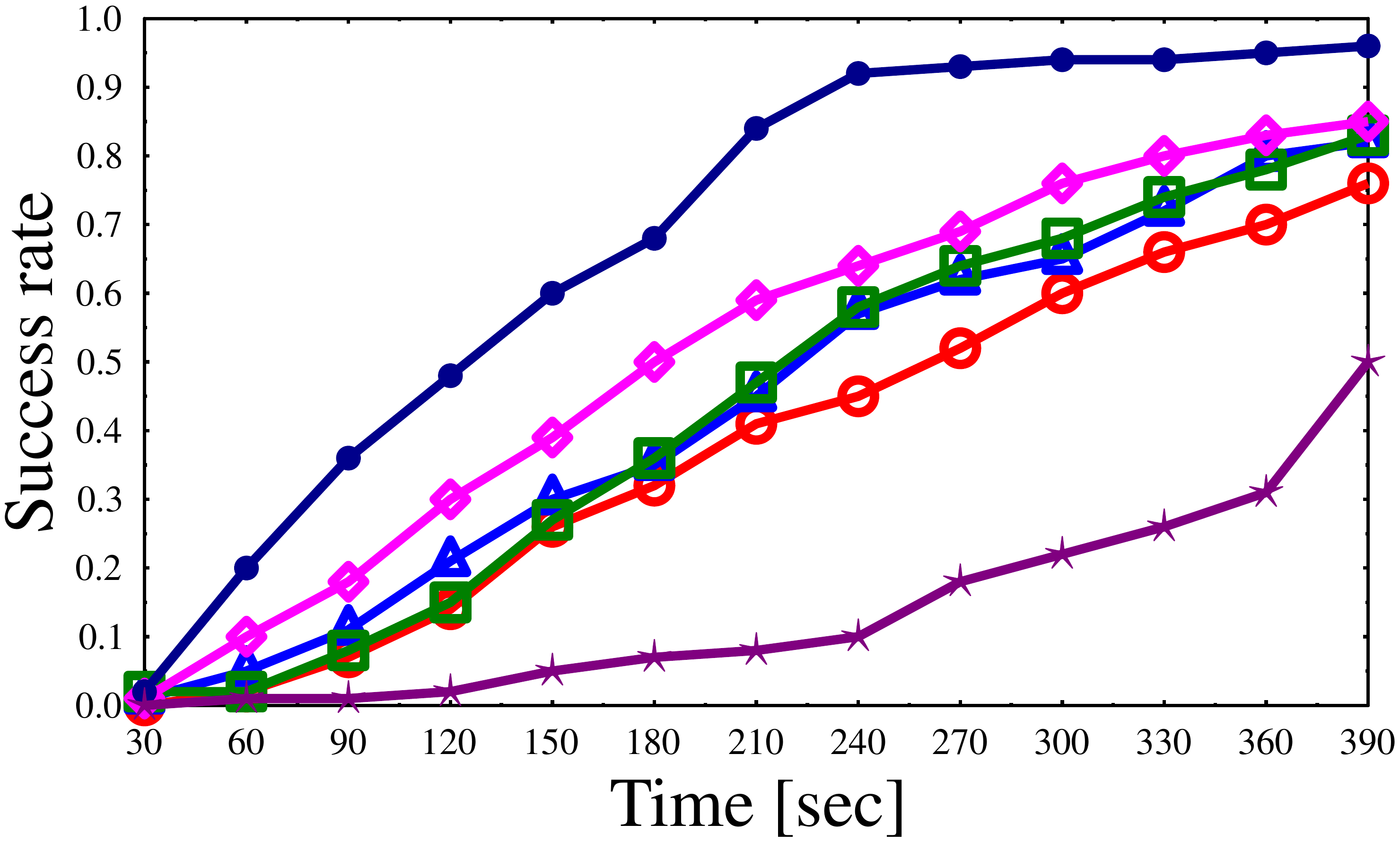}
   	\label{fig:cubicles_suc}
   }
   \hspace{1mm}
  \subfloat
   [\sf Legend]
   { 
   	\includegraphics[width=0.108\textwidth]
   		{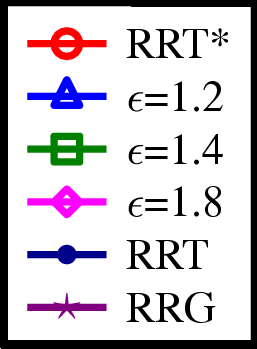}
   	\label{fig:legend_suc}
   }
  \caption{\sf 	Success rate for algorithms on different scenarios 
  							(RRT and RRT+RRT* have almost identical success rates; 
  							the plot for RRT+RRT* is omitted to avoid cluttering of the 
  							graph).}
  \label{fig:suc}
\end{figure*}

\begin{figure*}[t,b,h]
  \centering
  \subfloat
   [\sf Maze Scenario]
   { 
   	\includegraphics[width=0.25\textwidth]
   		{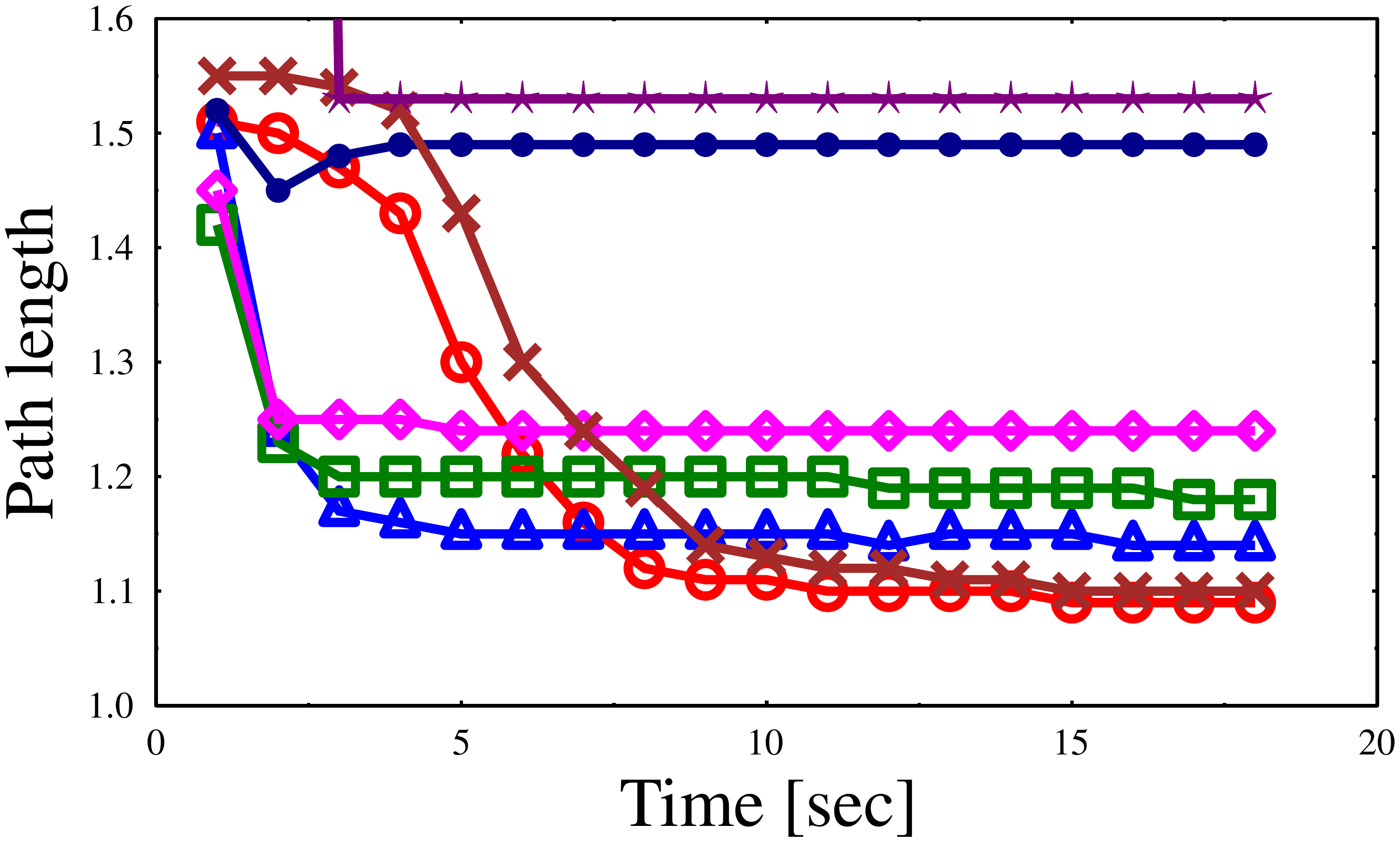}
   	\label{fig:maze_len}
   }
   \hspace{1mm}
  \subfloat
   [\sf Alternating barriers scenario]
   {
	\includegraphics[width=0.25\textwidth]
   		{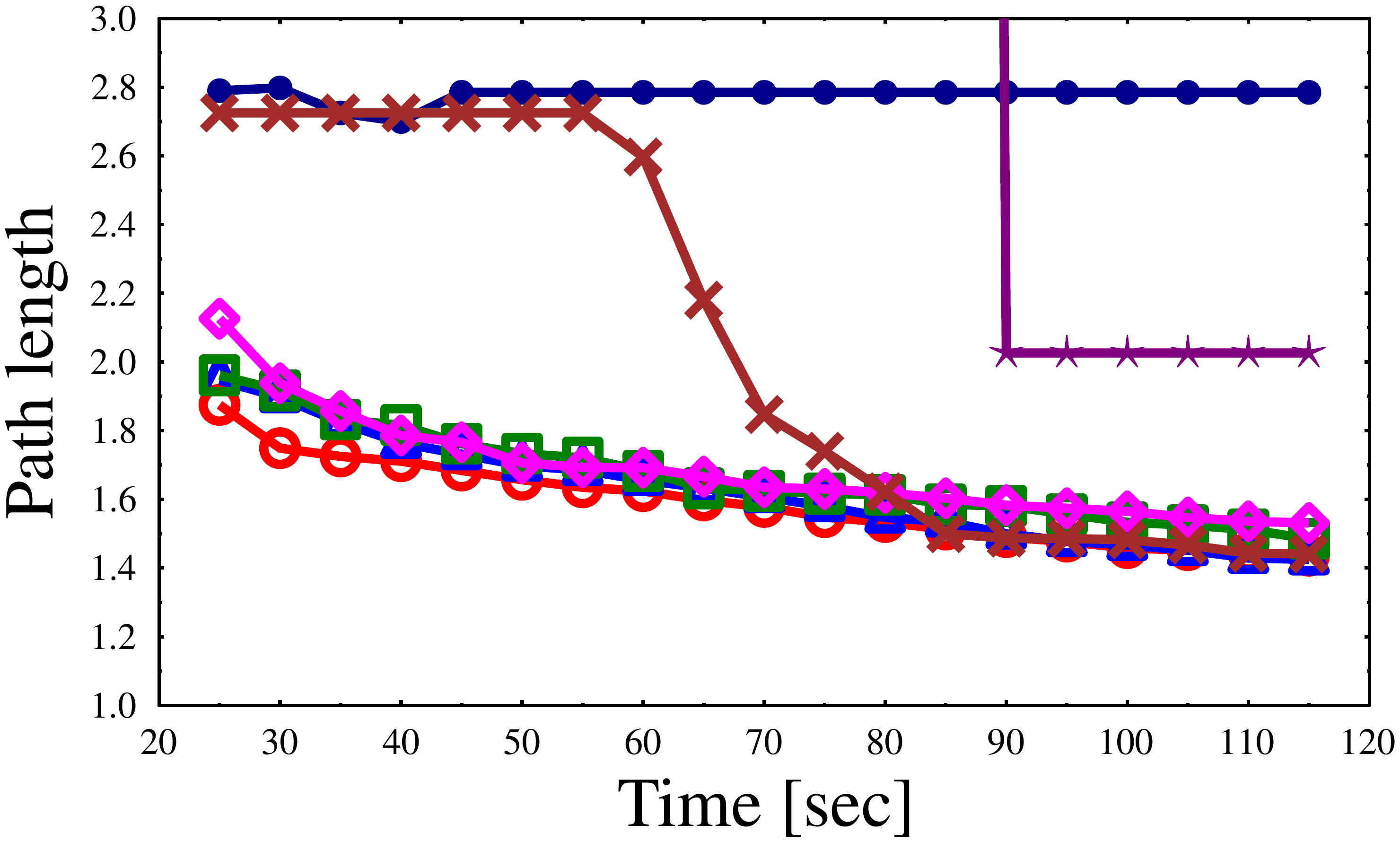}
   	\label{fig:alternating_len}
   }
   \hspace{1mm}
  \subfloat
   [\sf Cubicles scenario]
   { 
   	\includegraphics[width=0.25\textwidth]
   		{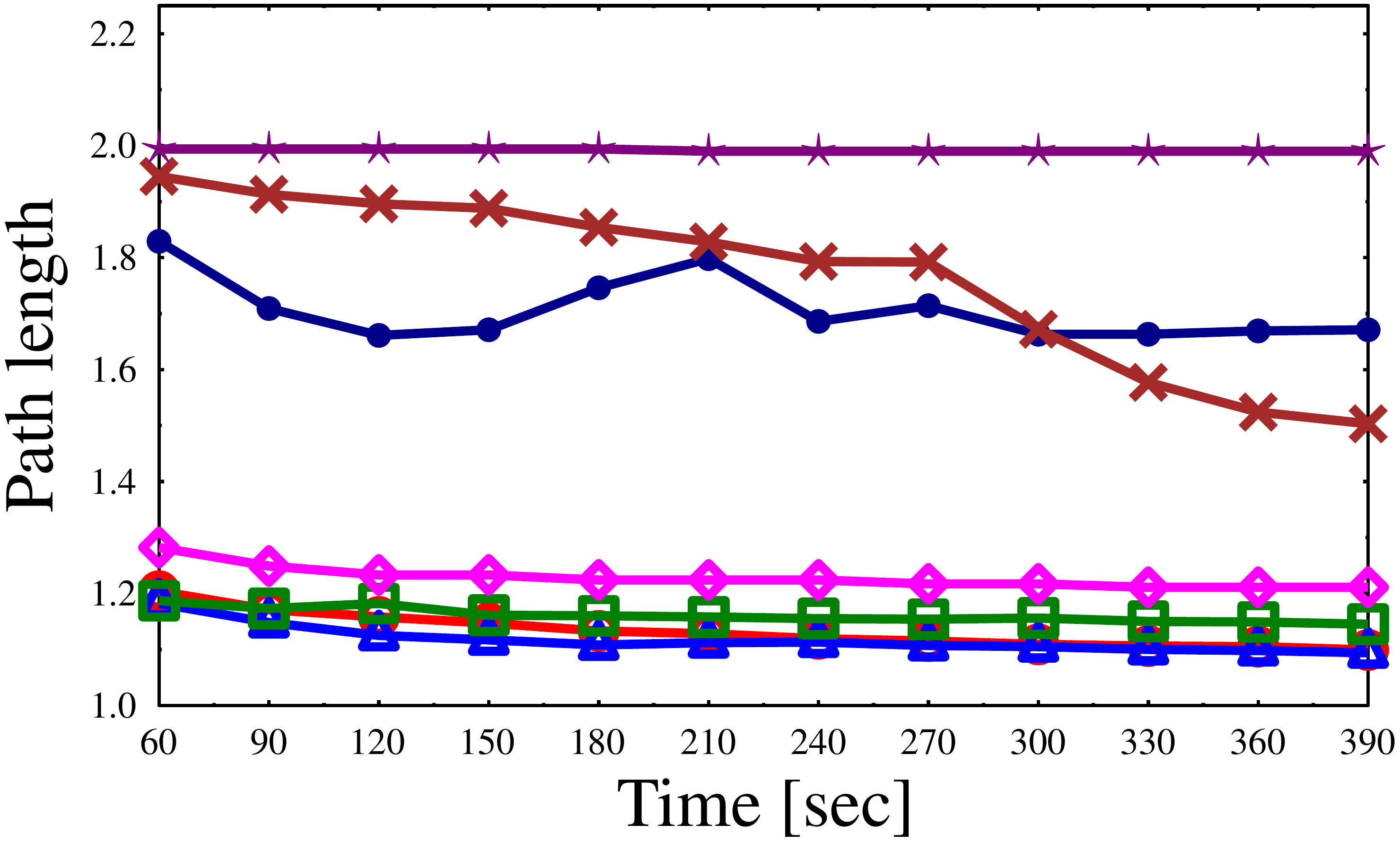}
   	\label{fig:cubicles_len}
   }
     \hspace{1mm}
  \subfloat
   [\sf Legend]
   { 
   	\includegraphics[width=0.108\textwidth]
   		{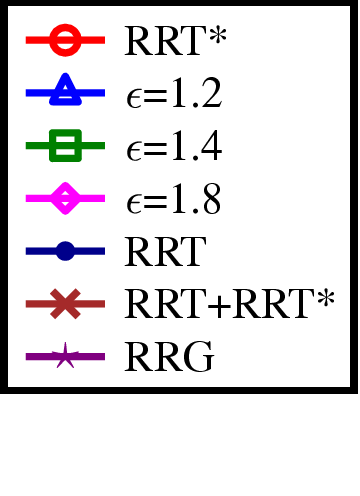}
   	\label{fig:legend_suc}
   } 
  \caption{\sf 	Path lengths for algorithms on different scenarios.
			  	Length values are normalized such that a length of one represents the length of an 
			  	optimal path.
  			}
  \label{fig:len}
  \vspace{-5mm}
\end{figure*}

Fig.~\ref{fig:suc} depicts similar behaviour for all scenarios:
As one would expect, the success rate for all algorithms has a monotonically increasing trend as the time budget increases.
For a specific time budget, the success rate for RRT and RRT+RRT* is typically highest while that of the RRT* and RRG is lowest. 
The success rate for \eRRT for a specific time budget, typically increases as the value of $\varepsilon$ increases. 
Fig.~\ref{fig:len} also depicts similar behavior for all scenarios: the average path length decreases for all algorithms 
(except for RRT).
The average path length for \eRRT typically decreases as the value of~$\varepsilon$ decreases and is comparable to that of RRT* for low values of $\varepsilon$.
RRT+RRT* behaves similarly to RRT* but with a ``shift'' along the time-axis which is due to the initial run of RRT. 
We note that although RRG and RRT+RRT* are asymptotically-optimal, their overhead makes them poor algorithms when one desires a \emph{high-quality} solution very fast.

Thus, Fig.~\ref{fig:suc} and~\ref{fig:len} should be looked at simultaneously as they encompass the tradeoff between speed to find \emph{any} solution and the quality of the solution found.
Let us demonstrate this on the alternating barriers scenario:
If we look at the success rate of each algorithm to find \emph{any} solution~(Fig.~\ref{fig:alternating_suc}),
one can see that RRT manages to achieve a success rate of 70\% after 30 seconds. RRT*, on the other hand, requires 70 seconds to achieve the same success rate (more than double the time). 
For all different values of $\varepsilon$, \eRRT manages  to achieve a success rate of 70\% after 50 seconds (around 60\% overhead when compared to RRT).
Now, considering the path length at 50 seconds, typically the paths extracted from \eRRT yield the same quality when compared to RRT* while ensuring a high success rate.

The same behavior of finding paths of high-quality (similar to the quality that RRT* produces) within the time-frames that RRT requires in order to find \emph{any} solution has been observed for both the Maze scenario and the Cubicles scenario.
Results omitted in this text.
For supplementary material the reader is referred to \url{http://acg.cs.tau.ac.il/projects/LBT-RRT}.


\section{Lazy, goal-biased \eRRT}
\label{sec:extensions}
In this section we show to further reduce the number of calls to the local planner by incorporating a lazy approach together with a goal bias.

\begin{algorithm}[t,b]
\caption{\texttt{consider\_edge\_goal\_biased}$(x_1, x_2)$}
\label{alg_update2}
\begin{algorithmic}[1]
	\STATE $c_{\min}^{\text{apx}} \leftarrow $
			\texttt{cost$_{\text{LPA*}}$}($\Tub$)
	\STATE	$\texttt{insert\_edge$_{\text{LPA*}}$}(\Glb, x_1, x_2)$  
	\STATE $x \leftarrow $
			\texttt{shortest\_path$_{\text{LPA*}}$}($\Glb$)
	\STATE $x_{parent} \leftarrow 
			\texttt{parent}_{\texttt{LPA*}}(\Glb, x)$
	\STATE $c_{\min}^{\text{lb}} \leftarrow $
			\texttt{cost$_{\text{LPA*}}$}($\Glb$)

	\WHILE {$c_{\min}^{\text{apx}} > (1 + \varepsilon) \cdot c_{\min}^{\text{lb}}$}
		
		\IF {(\texttt{collision\_free} ($x_{parent}, x$))}
			\STATE	$\texttt{insert\_edge$_{\text{LPA*}}$}(\Tub, x_{parent},x)$  
			\STATE  \texttt{shortest\_path$_{\text{LPA*}}$}($\Tub$)
				
			\STATE $c_{\min}^{\text{apx}} \leftarrow $
				\texttt{cost$_{\text{LPA*}}$}($\Tub$)
			\STATE $x \leftarrow 
				\texttt{parent}_{\texttt{LPA*}}(x)$
       	\ELSE
			\STATE \texttt{delete\_edge$_{\text{LPA*}}$}$(\Glb, x_{parent}, x)$
			\STATE \textbf{GoTo} line 3
		\ENDIF
	\ENDWHILE

\end{algorithmic}
\end{algorithm}						
\eRRT maintains the lower bound invariant to \emph{every} node.
This is desirable in settings where a high-quality path to every point in the configuration space is required.
However, when only a high-quality path to the goal is needed, 
this may lead to unnecessary time-consuming calls to the local planner.

Therefore, we suggest the following variant of \eRRT where 
we relax the bounded approximation invariant such that it holds only for nodes $x \in X_{goal}$. 
This variant is similar to \eRRT but
differs with respect to the calls to the local planner and with respect to the dynamic shortest-path algorithm used.
As we only maintain the bounded approximation invariant to the goal nodes, 
we do not need to continuously update the (approximate) shortest path to every node in~$\Glb$.
We replace the SSSP algorithm, 
which allows to compute the shortest paths to \emph{every} node in a dynamic graph, with 
Lifelong Planning A* (LPA*)~\cite{KLF04}.
LPA*
allows to repeatedly find shortest paths from a given
start to a given goal while allowing for edge insertions and deletions.
Similiar to A*~\cite{P84}, this is done by using heuristic function~$h$ 
such that for every node~$x$, 
$h(x)$ is an  estimate of the cost to reach the goal from~$x$.

Given a start vertex $x_{init}$, a goal region $X_{Goal}$, 
we will use the following functions which are provided when implementing LPA*:
\texttt{shortest\_path$_{\text{LPA*}}$}$(G)$,
recomputes the shortest path to reach $X_{Goal}$ from $x_{init}$ on the graph $G$
and returns the node $x \in X_{Goal}$ 
such that 
$x \in X_{Goal}$ and 
$\texttt{cost}_{\Glb}(x)$ is minimal among all $x' \in X_{Goal}$.
Once the function has been called, 
the following functions take constant running time:
\texttt{cost$_{\text{LPA*}}$}$(G)$ 
returns the minimal cost to reach $X_{Goal}$ from $x_{init}$ on the graph $G$
and for every node $x$ lying on a shortest path to the goal, 
\texttt{parent}$_{\texttt{LPA*}}(G, x)$ 
returns the predecessor of the node~$x$ along this path.
Additionally, 
$\texttt{insert\_edge$_{\text{LPA*}}$}(G, x, y)$  
and
\texttt{delete\_edge$_{\text{LPA*}}$}$(G, x, y)$ 
inserts (deletes) the edge $(x,y)$ to (from) the graph $G$, respectively.

We are now ready to describe Lazy, goal-biased \eRRT which is similar to \eRRT 
except for the way new edges are considered. 
Instead of the function \texttt{consider\_edge} 
called in lines~14 and~16 of Alg.~\ref{alg_RRT},
the function \texttt{consider\_edge\_goal\_biased} is called.

\texttt{consider\_edge\_goal\_biased}$(x_1, x_2)$, outlined in Alg.~\ref{alg_update2},
begins by computing the cost to reach the goal in~$\Tub$ (line~1)
and in~$\Glb$ after adding the edge $(x_1, x_2)$ lazily to~$\Glb$ (lines~2-5).
Namely, the edge is added with no call to the local planner and without checking if the bounded approximation invariant is violated.
Note that the \emph{relaxed} bounded approximation invariant is violated (line~6)
only if a path to the goal is found.
Clearly, if all edges along the shortest path to the goal are found to be collision free, then the invariant holds.
Thus, the algorithm attempts to follow the edges along the path 
(starting at the last edge and backtracking towards~$x_{init}$)
one by one and test if they are indeed collision-free.
If an edge is collision free (line~7),
it is inserted to~$\Tub$ (line~8), 
and a path to the goal in~$\Tub$ is recomputed (line~9).
This is repeated as long as the relaxed bounded approximation invariant is violated.
If the edge is found to be in collision (line~12),
it is removed from~$\Glb$ (line~13) and the process is repeated (line~14).

Following similar arguments as described in Section~\ref{sec:alg}, 
one can show the correctness of the algorithm.
We note that as long as no path has been found, 
the algorithm performs no more calls to the local planner than RRT.
Additionally, it is worth noting that the planner bares resemblance with Lazy-RRG*~\cite{H15}.

We compared Lazy, goal-biased \eRRT with \eRRT, RRT* and RRG on the Home scenario (Fig.~\ref{fig:home_scene}).
In this scenario, a low quality solution is typically easy to find and all algorithms (except RRG) find a solution with roughly the same success rate as RRT (results omitted). 
Converging to the optimal solution requires longer running times as low-quality paths are easy to find yet high-quality ones pass through narrow passages.
Fig.~\ref{fig:home_cost} depicts the path length obtained by the algorithms as a function of time.
The convergence to the optimal solution of RRG is significantly slower than all other algorithms.
Both \eRRT and RRT* find a low quality solution 
(between five and six times longer than the optimal solution)
within the allowed time frame and manage to slightly improve upon its cost 
(with RRT* obtaining slightly shorter solutions than \eRRT).
When enhancing \eRRT with a lazy approach together with goal-biasing, one can observe that the convergence rate improves substantially.


\begin{figure*}[t,b,h]
  \centering
  \subfloat
   [\sf Home scenario]
   { 
   	\includegraphics[height=4cm]
   		{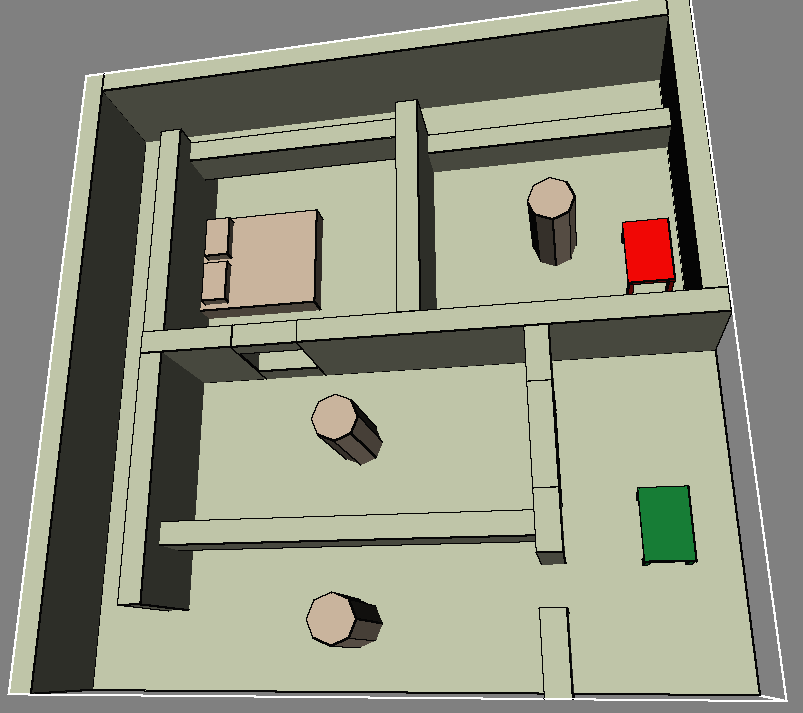}
   	\label{fig:home_scene}
   }
   \hspace{5mm}
  \subfloat
   [\sf Solution cost]
   { 
   	\includegraphics[height=4cm]
   		{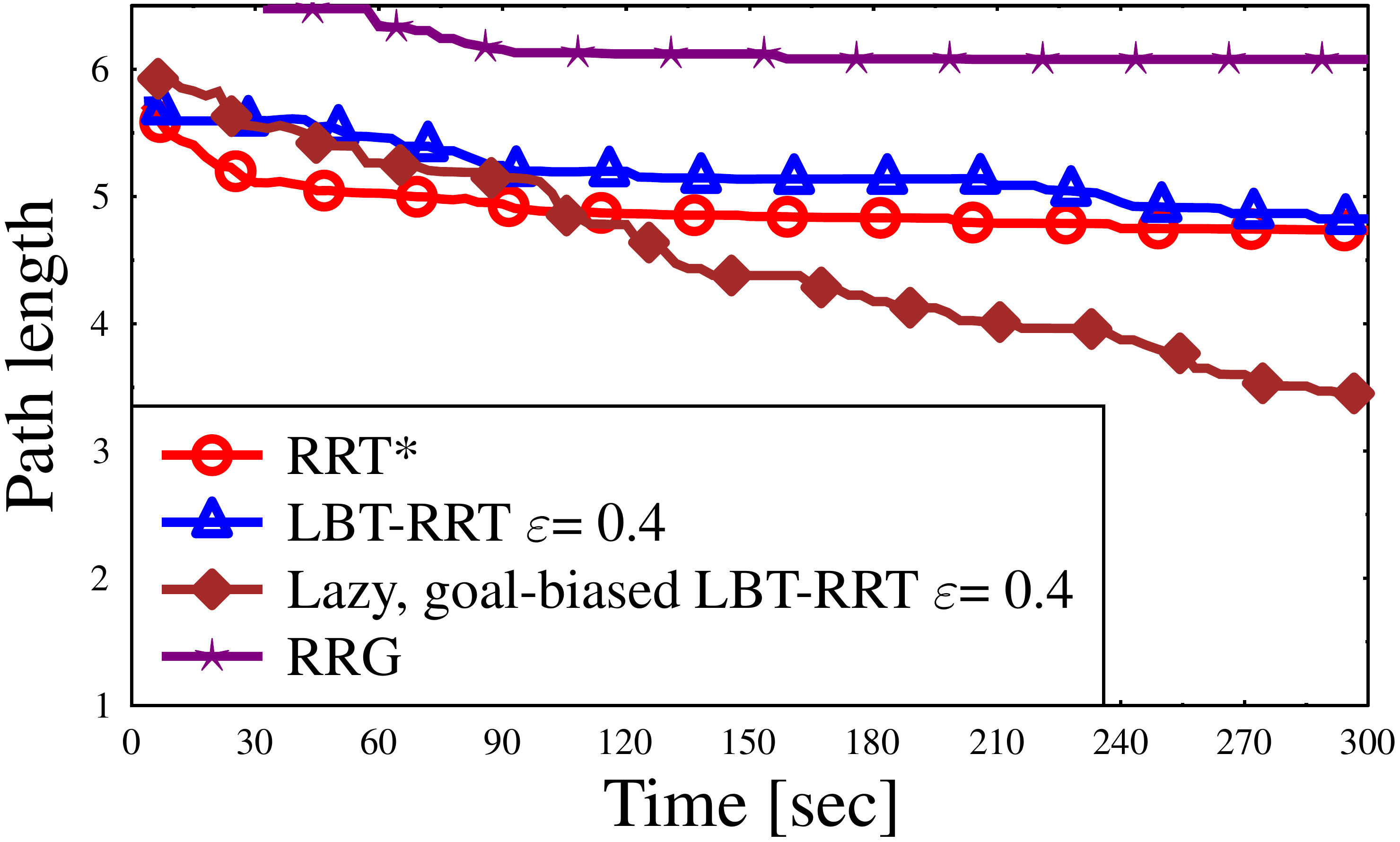}
   	\label{fig:home_cost}
   }
  \caption{\sf 	Simulation results comparing 
  				Lazy, goal-biased \eRRT with \eRRT, RRT* and RRG.
  				(a) Home scenario (provided by the OMPL distribution). 
  				Start and target table-shaped robots are depicted in 
  				green and red, respectively.
  				(b) Path lengths as a function of computation time.
  				Length values are normalized such that a length of one 
  				represents the length of an optimal path
  			}
  \label{fig:len}
  \vspace{-5mm}
\end{figure*}
\section{Conclusion and future work}
\label{sec:con}
In this work we presented an asymptotically near-optimal motion planning algorithm.
Using an approximation factor allows the algorithm to avoid calling the computationally-expensive local planner when no substantially better solution may be obtained.
\eRRT, together with the lazy, goal-biased variant, 
make use of \emph{dynamic shortest path algorithms}.
This is an active research topic in many communities such as 
artificial intelligence and communication networks.

Hence, 
the algorithms we proposed in this work
may benefit from
any advances made for dynamic shortest path algorithms.
For example, recently D'Andrea et al.~\cite{DDFLP13} presented an algorithm that allows for dynamically maintaining shortest path trees under \emph{batches} of updates which can be used by \eRRT instead of the SSSP algorithm.
\textVersion{In many applications where either RRT or RRT* were used we argue that \eRRT may serve as a superior alternative with no fundamental modification to the underlying algorithms,
which originally use RRT or RRT*. 
Moreover, one may consider alternative implementations of \eRRT using tools developed for either RRT or RRT* that can enhance \eRRT.
For further discussions on the subject, 
see the expanded arXiv version of this paper~\cite{SH14-arxiv}.

We show~\cite{SH14-arxiv} that the framework presented in this paper for relaxing the optimality of RRG can be used to have a similar effect on another asymptotically-optimal sampling-based algorithm, FMT*~\cite{JP13}. 
It would be interesting to see if other algorithms can benefit from the approach we take in this work.
}{}

Looking to further extend our framework, we seek natural stopping criteria for \eRRT.
Such criteria could possibly be related to the rate at which the quality is increased as additional samples are introduced.
Once such a criterion is established, one can think of the following  framework:
Run \eRRT with a large approximation factor (large $\varepsilon$)
, once the stopping criterion has been met, decrease the approximation factor and continue running.
This may allow an even quicker convergence to find any feasible path while allowing for refinement as time permits (similar to~\cite{APD11}).
While changing the approximation factor in \eRRT may possibly require a massive rewiring of~$\Glb$ (to maintain the bounded approximation invariant) this is not the case in Lazy, goal-biased \eRRT.
In this variant of \eRRT the approximation factor can change at any stage of the algorithm without any modifications at all.

An interesting question to be further studied is can our framework be applied to different quality measures.
For certain measures, such as bottleneck clearance of a path, this is unlikely, as bounding the quality of an edge already identifies if it is collision-free. 
However, for some other measures such as energy consumption, we believe that the framework could be effectively used.

\section{Acknowledgements}\label{sec:ack}
We wish to thank
Leslie Kaelbling
for suggesting the RRT+RRT* algorithm,
Chengcheng Zhong
for feedback on the implementation of the algorithm
and
Shiri Chechik for advice regrading dynamic shortest path algorithms.

\textVersion{}{
\appendix
\section{Additional applications \& variants}
\label{sec:app}
RRT has been used in numerous applications and various efficient implementations and heuristics have been suggested for it.
Even the relatively recent RRT* has already gained many applications and various implementations.  
Typically, the applications rely on the efficiency of RRT or the asymptotic optimality of RRT*.
We list two such applications (Sections~A and~B below) and discuss the possible advantage of replacing the existing planner (either RRT or RRT*) with \eRRT.

Efficient implementations and heuristics typically take into account the primitive operations used by the RRT and the RRT* algorithms (such as collision detection, nearest neighbor computation, sampling procedure etc.).
Thus, techniques suggested for efficient implementations of RRT and RRT* may be applied to \eRRT with little effort as the latter relies on the same primitive operations. We give two examples in Sections~C and~D below.

Finally, we show how to apply our approach to a different asymptotically-optimal sampling-based algorithm---Fast Marching Trees (FMT*)~\cite{JP13}.

\subsection{Re-planning using RRTs:}
Many real-world applications involve a \Cs that undergoes changes (such as moving obstacles or partial initial information of the workspace).
A common approach to plan in such dynamic environments is to run RRT, and re-plan when a change in the environment occurs.
Re-planning may be done from scratch although this can be unnecessary and time consuming as the assumption is that only part of the environment changes.
Ferguson et al.~\cite{FS06} suggest to 
(i) plan an initial path using RRT,
(ii) when a change in the configuration space is detected, nodes in the existing tree may be invalidated and a ``trimming'' procedure is applied where invalid parts are removed and
(iii) the trimmed tree is grown until a new solution is generated.

Obviously \eRRT can replace RRT in the suggested scheme. 
If the overhead of running \eRRT when compared to RRT is acceptable (which may indeed be the case as the experimental results in Section~\ref{sec:eval} suggest), then the algorithm will be able to produce \emph{high-quality} paths in dynamic environments.

\subsection{High-quality planning on implicitly-defined manifolds:}
Certain motion-planning problems, such as grasping with a multi-fingered hand, involve planning on implicitly-defined manifolds.
Jaillet and Porta~\cite{JP12} address the central challenges of applying RRT* to such cases. 
The challenges include sampling uniformly on the manifold, locating the nearest neighbors using the metric induced by the manifold, computing the shortest path between two points and more.
They suggest AtlasRRT*, an adaptation of RRT* that operates on manifolds.
It follows the same structure as RRT* but maintains an atlas by iteratively adding charts to the atlas to facilitate the primitive operations of RRT* on the manifold (i.e., sampling, nearest-neighbor queries, local planning etc.).

If one is concerned with \emph{fast convergence} to a high quality solution, 
\eRRT can be used seamlessly, replacing the guarantee for optimality with a weaker near-optimality guarantee.

\subsection{Sampling Heuristics:}
Following the exposition of  RRT*, Akgun and Stilman~\cite{AS11} suggested a sampling bias for the RRT* algorithm. 
This bias accelerates cost decrease of the path to the goal in the RRT* tree.
Additionally, they suggest a simple node-rejection criterion to increase efficiency.
These heuristics may be applied to the \eRRT by simply changing the procedure \texttt{sample\_free} (Alg.~\ref{alg_RRT}, line 3).

\subsection{Parallel RRTs:}
In recent years, hardware allowing for parallel implementation of existing algorithms has become widespread both in the form of multi-core Central Processing Units (CPUs) and in the form of Graphics Processing Units (GPUs). 
Parallel implementations for sampling based algorithms have already been proposed in the late 1990s~\cite{AD99}. Since then, a multitude of such implementations emerged (see, e.g.,~\cite{IA12, BKF11}  for a detailed literature review).

We review two approaches to parallel implementation of RRT and RRT* and claim that both approaches may be used for parallel implementation of \eRRT.
The first approach, by Ichnowski et al.~\cite{IA12} suggests parallel variants of RRT and RRT* on multi-core CPUs that achieve superlinear speedup. 
By using CPU-based implementation, their approach retains the ability
to integrate the planners with existing CPU-based libraries and algorithms. They achieve superlinear speedup by:
(i) lock-free parallelism using atomic operations to reduce slowdowns caused by contention,
(ii) partition-based sampling to reduce the size of each processor core's working data set and to improve cache efficiency and
(iii) parallel backtracking in the rewiring phase of RRT*.
\eRRT may benefit from all three key algorithmic features.

Bialkowski et al.~\cite{BKF11} present a second approach for parallel implementation of RRT and RRT*.
They suggest a massively parallel, GPU-based implementation of the collision-checking procedures of RRT and RRT*. 
Again, this approach may be applied to the collision-checking procedure of \eRRT without any need for modification.

\subsection{Framework Extensions}
\label{sec:fmt}
FMT*, proposed by Janson and Pavone, is a recently introduced asymptotically-optimal algorithm which is shown to converge to an optimal solution faster than PRM* or RRT*.
It uses a set of probabilistically-drawn configurations to construct a tree, which grows in cost-to-come space. Unlike RRT*, it is a batch algorithm that works with a predefined number of nodes~$n$.

We first describe the FMT* algorithm (outlined in Alg.~\ref{alg:fmt}),
we continue to describe how to apply it in an anytime fashion and conclude by describing how to apply our framework to this anytime variant.
FMT* samples $n$ collision-free nodes\footnote{By slight abuse of notation, \texttt{sample\_free($n$)} is a procedure returning~$n$ collision-free samples.} $V$ (line~1) and 
builds a minimum-cost spanning tree rooted at the initial configuration by maintaining two sets of nodes $H, W$ such that 
$H$ is the set of nodes added to the tree that may be expanded and $W$ is the set of nodes not in the tree~(line~2).
It then computes for each node the set of nearest neighbors\footnote{The nearest-neighbor computation can be delayed and performed only when needed but we present the batched mode of computation to simplify the exposition.} of radius~$r(n)$~(line~4).
The algorithm repeats the following process: the node $z$ with the lowest cost-to-come value is chosen from $H$ (lines 5 and 17).
For each neighbor $x$ of $z$ that is not already in $H$, the algorithm finds its neighbor $y \in H$ such that the cost-to-come of $y$ added to the distance between $y$ and $x$ is minimal~(lines~7-10).
If the local path between $y$ and $x$ is free, $x$ is added to~$H$ with $y$ as its parent~(lines~11-13).
At the end of each iteration $z$ is removed from~$H$~(line~14).
The algorithm runs until a solution is found or there are no more nodes to process.

We now outline a straightforward enhancement to FMT* to make it anytime.
As noted in previous work (see, e.g.,~\cite{WBC13}) one can turn a batch
algorithm into an anytime one by the following (general) approach: choose an
initial (small) number of samples $n = n_0$ and apply the algorithm. 
As long as time permits, double $n$ and repeat the process. 
We call this version anytime FMT* or aFMT*, for short.

We can further speed up this method by reusing both existing samples and connections from previous iterations.
Assume the algorithm was run with $n$ samples and now we wish to re-run it with $2n$ samples.
In order to obtain the $2n$ random samples, we take the $n$ random samples from the previous iteration together with $n$ new additional random samples.
For each node that was used in iteration $i-1$, on average half of its neighbors in iteration $i$ are nodes from iteration $i-1$ and half of its neighbors are newly-sampled nodes.
Thus, if we cache the results of calls to the local planner, we can use them in future iterations using the framework presented in this paper. We call this algorithm LBT-aFMT*.

Alg.~\ref{alg:near_optimal_fmt} outlines one iteration of LBT-aFMT*, differences between FMT* and LBT-aFMT* are colored in red. 
Similar to LBT-RRT, LBT-aFMT* constructs two trees $\calT_{lb}$ and $\calT_{ub}$ and maintains the bounded approximation invariant and the lower bound invariant. 
The two invariants are maintained by using a cache that can efficiently answer if the local path between two configurations is collision-free for a subset of the nodes used.
The proof of near-optimality of LBT-aFMT*, which is omitted, follows the same lines as the analysis presented in Section~\ref{susbsec:analysis}.

In order to demonstrate the effectiveness of applying our lower-bound  framework to aFMT*, we compared the two algorithms, namely aFMT* and LBT-aFMT*, in the case of a three-dimensional configuration space which we call the Corridors scenario (Fig.~\ref{fig:corr}). It consists of a planar hexagon  robot that can translate and rotate amidst a collection of small obstacles. 
There are two main corridors from the start to the goal position, a wide one and a narrow one.
A relatively small number of samples suffices to find a path through the wide corridor, yet in order to compute a low-cost path through the narrow corridor, many more samples are needed.
This demonstrates how LBT-aFMT* refrains from refining an existing solution when no significant advantage can be obtained.
Fig.~\ref{fig:corr_success} depicts the success rate of finding a path through the narrow corridor as time progresses.
Clearly, for these types of settings LBT-aFMT* performs favorably over aFMT* even for large approximation factors.
For example, to reach a $60\%$ succes rate in finding a path through the narrow corridor, LBT-aFMT*, run with $\varepsilon = 0.5$, needs half of the time needed by aFMT*.

\begin{figure*}[t,b]
  \centering
  \subfloat
   [\sf ]
   { 
   	\includegraphics[height =4.25 cm]{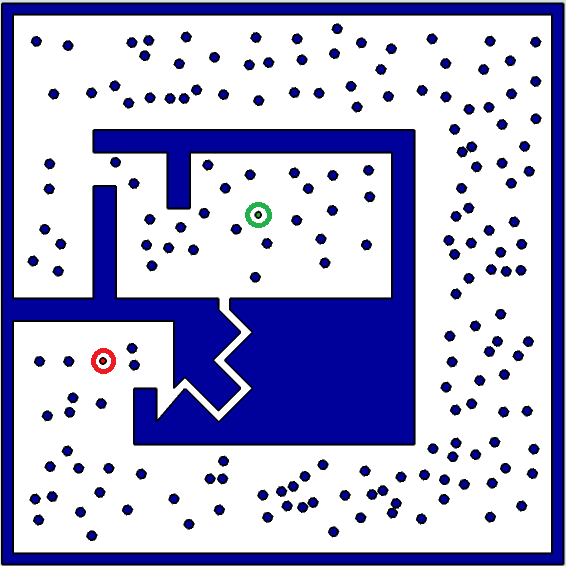}
   	\label{fig:corr}
   }
  \subfloat
   [\sf ]
   {
   	\includegraphics[height =4.25 cm]{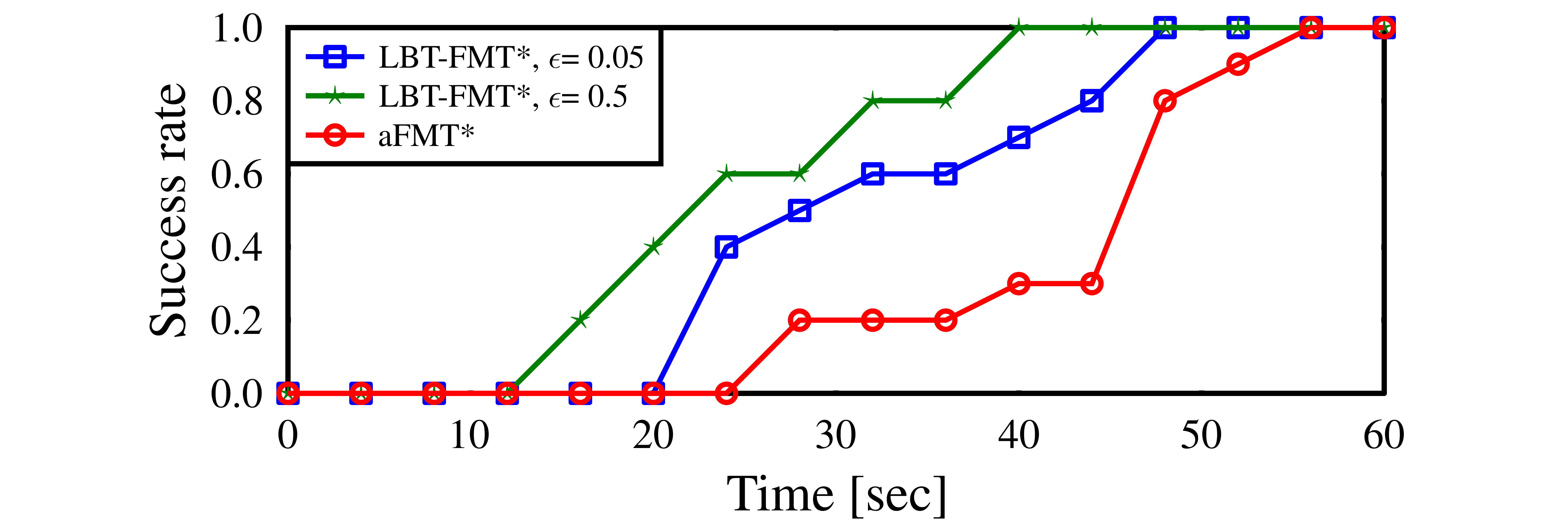}
   	\label{fig:corr_success}
   }
  \caption{\sf 	(a) Corridors scenario---Start and target configurations 
  							are depicted by green and  red circles, respectively.
  							The narrow corridor appears in the lower half of the workspace.
  							(b) Success rate to find a path through the narrow corridor.}
  \label{fig:corridor_figs}
	\vspace{-3mm}
\end{figure*}

\begin{algorithm}[t, b]
\caption{\texttt{fast\_marching\_tree}}
\label{alg:fmt}
\begin{algorithmic}[1]
	\STATE	$V \leftarrow \set{x_{\text{init}}} \cup \texttt{sample\_free}(n)$;
					\hspace{1mm}
					$E \leftarrow \emptyset$;
					\hspace{1mm}
					$\calT\leftarrow (V,E)$
  \STATE	$W \leftarrow V \setminus \set{x_{\text{init}}}$;
  				\hspace{3mm}
					$H \leftarrow \set{x_{\text{init}}}$
  \FORALL{$v \in V$} 
  	\STATE	$N_v \leftarrow 
  					\texttt{nearest\_neighbors}(V \setminus \set{v}, v, r(n))$
  \ENDFOR
  

  \STATE	$z \leftarrow x_{\text{init}}$
	\WHILE {$z \notin \calX_{\text{Goal}}	$}
		\STATE $H_{\text{new}} \leftarrow \emptyset$;
					 \hspace{3mm}
					 $X_{\text{near}} \leftarrow W \cap N_z$
		

		\FOR  {$x \in X_{\text{near}}$}
			\STATE $Y_{\text{near}} \leftarrow H \cap N_x$
							\hspace{3mm}
			\STATE $y_{\text{min}} \leftarrow \arg \min_{y \in Y_{\text{near}}} 
								\set{\texttt{cost}_{\calT}(y) + \texttt{cost}(y,x)} $
								
			
			\IF {\texttt{collision\_free}$(y_{\text{min}}, x)$}
				\STATE $\calT.\texttt{parent}(x) \leftarrow y_{\text{min}}$
				\STATE $H_{\text{new}} \leftarrow H_{\text{new}} \cup \set{x}$;
							 \hspace{3mm}
					 		 $W \leftarrow W \setminus \set{x}$
			\ENDIF
		\ENDFOR
								
					
		\STATE $H \leftarrow (H \cup H_{\text{new}}) \setminus \set{z}$
		\IF {$H = \emptyset$}
			\RETURN FAILURE
		\ENDIF
		
		\STATE $z \leftarrow \arg \min_{y \in H} 
								\set{\texttt{cost$_{\calT}$}(y)} $
														
	\ENDWHILE

	\RETURN PATH \hspace{3mm}
\end{algorithmic}
\end{algorithm}

\begin{algorithm}[tb]
\caption{\texttt{LBT-aFMT$^*$ (Cache)}}
\label{alg:near_optimal_fmt}
\begin{algorithmic}[1]
	\STATE	$V \leftarrow \set{x_{\text{init}}} \cup \texttt{sample\_free}(n)$;
					\hspace{1mm}
					$E \leftarrow \emptyset$;
	{\color{red}
	\STATE	$\calT_{lb} \leftarrow (V,E)$;
					\hspace{1mm}
					$\calT_{apx} \leftarrow (V,E)$
	}
  \STATE	$W \leftarrow V \setminus \set{x_{\text{init}}}$;
  				\hspace{3mm}
					$H \leftarrow \set{x_{\text{init}}}$
  \FORALL{$v \in V$} 
  	\STATE	$N_v \leftarrow 
  					\texttt{nearest\_neighbors}(V \setminus \set{v}, v, r(n))$
  \ENDFOR
  

  \STATE	$z \leftarrow x_{init}$
	\WHILE {$z \notin \calX_{Goal}	$}
		\STATE $H_{\text{new}} \leftarrow \emptyset$;
					 \hspace{3mm}
					 $X_{\text{near}} \leftarrow W \cap N_z$
		

		\FOR  {$x \in X_{near}$}
			\STATE $Y_{near} \leftarrow H \cap N_x$
							\hspace{3mm}
			{\color{red}
			\STATE $y_{lb} \leftarrow \arg \min_{y \in Y_{near}} 
								\set{\texttt{cost}_{\Tlb}(y) + \texttt{cost}(y,x)} $
			\STATE $Y_{apx} \leftarrow 
								\set{y \in Y_{near} \ | \ (y,x) \in \texttt{Cache)} }$
			\STATE $y_{apx} \leftarrow \arg \min_{y \in Y_{apx}} 
								\set{\texttt{cost}_{\Tub}(y) + \texttt{cost}(y,x) }$
		

			\STATE $c_{lb} \leftarrow \texttt{cost}_{\Tlb}(y_{lb}) + 
																 \texttt{cost}(y_{lb},x) $
			\STATE $c_{apx} \leftarrow \texttt{cost}_{\Tub}(y_{apx}) + 
																 \texttt{cost}(y_{apx},x) $

			\IF {$c_{apx}
							\leq 
					 (1+\varepsilon) \cdot  c_{lb}$}
				\STATE $\Tlb.\texttt{parent}(x) \leftarrow y_{lb}$
				\STATE $\Tub.\texttt{parent}(x) \leftarrow y_{apx}$
			
					{\color{black}			
					\STATE $H_{new} \leftarrow H_{new} \cup \set{x}$;
							 	 \hspace{3mm}
							   $W \leftarrow W \setminus \set{x}$					}				
			\ELSE
				\IF {\texttt{collision\_free}$(y_{lb}, x)$}
					\STATE $\Tlb.\texttt{parent}(x) \leftarrow y_{lb}$
					\STATE $\Tub.\texttt{parent}(x) \leftarrow y_{lb}$
				
					{\color{black}			
					\STATE $H_{new} \leftarrow H_{new} \cup \set{x}$;
							 	 \hspace{3mm}
							   $W \leftarrow W \setminus \set{x}$					}
				\ENDIF
			\ENDIF
			}
		\ENDFOR
		
		\vspace{2mm}
					
		\STATE $H \leftarrow (H \cup H_{new}) \setminus \set{z}$
		\IF {$H = \emptyset$}
			\RETURN FAILURE
		\ENDIF
		
		\STATE $z \leftarrow \arg \min_{y \in H} 
								\set{\texttt{cost}(y)} $
						
	\ENDWHILE

	\RETURN PATH 
\end{algorithmic}
\end{algorithm}						

}


\begin{IEEEbiography}[{\includegraphics[width=1in,height=1.25in,clip,keepaspectratio]{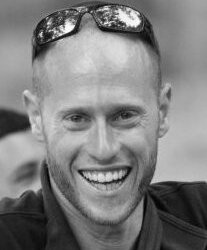}}]
{Oren Salzman} is a PhD-student at 
the School for Computer Science, Tel-Aviv University, Tel Aviv 69978, ISRAEL.
\end{IEEEbiography}

\begin{IEEEbiography}[{\includegraphics[width=1in,height=1.25in,clip,keepaspectratio]{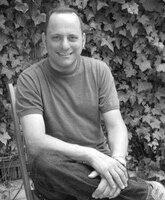}}]
{Dan Halperin} is a Professor at 
the School for Computer Science, 
Tel-Aviv University, Tel Aviv 69978, ISRAEL.
\end{IEEEbiography}

\end{document}